\documentclass[runningheads]{llncs}
\usepackage{textcomp}
\usepackage[T1]{fontenc}
\usepackage{graphicx}
\usepackage{amsthm}
\usepackage{amsmath}
\newtheorem{theoremrep}{Theorem}

\newcommand{\R}{\mathbb{R}}

\newcommand{\uu}{\mathbf{u}}
\newcommand{\vv}{\mathbf{v}}

\usepackage{amssymb}
\usepackage{multirow}
\usepackage{graphicx}
\usepackage{subcaption}
\usepackage{dsfont}
\usepackage{wrapfig}
\usepackage{pdflscape}

\usepackage{caption}
\usepackage{float}
\usepackage{capt-of}

\newcommand{\twoplotblock}[4]{%
  \begin{minipage}[t]{0.46\textwidth}
    \centering
    \includegraphics[width=\linewidth]{#1}\par\vspace{0.1em}
    \includegraphics[width=\linewidth]{#2}\par\vspace{0.1em}
    \refstepcounter{subfigure}%
    \text{(\thesubfigure) #3}%
    \label{#4}
  \end{minipage}
}

\usepackage[colorinlistoftodos]{todonotes}

\usepackage[utf8]{inputenc}

\usepackage{mathtools}
\usepackage{tikz}
\usetikzlibrary{automata,positioning,angles,quotes}
\usetikzlibrary{tikzmark}
\usetikzlibrary{arrows.meta}

\usepackage{wrapfig}

\usepackage{booktabs}  
\usepackage{array}     

\newcommand{\citep}[1]{{\cite{#1}}}
\usepackage[english]{babel}
\usepackage{upgreek}
\usepackage{authblk}
\usepackage{ dsfont }
\usepackage{gensymb}
\usepackage[dvipsnames]{xcolor}
\usepackage{fontawesome5}
\usepackage{tikz}
\usetikzlibrary{automata,positioning,angles,quotes,scopes, calc, fadings}
\usetikzlibrary{arrows.meta}
\usepackage{pgfplots}
\usepgfplotslibrary{fillbetween}
\pgfplotsset{compat=newest}
\usepgfplotslibrary{colormaps} 

\usepackage{amssymb}
\usepackage{multirow}
\usepackage{graphicx}
\usepackage{adjustbox}
\usepackage{subcaption}
\usepackage[colorinlistoftodos]{todonotes}

\definecolor{lightred}{RGB}{241, 225, 222}
\definecolor{color1}{rgb}{0.1,0.498039215686275,0.9549019607843137}
\definecolor{alizarin}{rgb}{0.82, 0.1, 0.26}
\definecolor{antiquewhite}{rgb}{0.98, 0.92, 0.84}
\definecolor{azure}{rgb}{0.94, 1.0, 1.0}
\definecolor{offwhite}{rgb}{0.98, 0.95, 0.95}
\definecolor{pigment}{rgb}{0.2, 0.2, 0.6}

\usepackage{xcolor}

\definecolor{lightred}{RGB}{241, 225, 222}
\definecolor{color1}{rgb}{0.1,0.498039215686275,0.9549019607843137}
\definecolor{alizarin}{rgb}{0.82, 0.1, 0.26}
\definecolor{antiquewhite}{rgb}{0.98, 0.92, 0.84}
\definecolor{azure}{rgb}{0.94, 1.0, 1.0}
\definecolor{offwhite}{rgb}{0.98, 0.95, 0.95}
\definecolor{pigment}{rgb}{0.2, 0.2, 0.6}

\usepackage[utf8]{inputenc} 
\usepackage[T1]{fontenc}    
\usepackage{url}            
\usepackage{booktabs}       
\usepackage{amsfonts}       
\usepackage{nicefrac}       
\usepackage{microtype}      
\usepackage{stackengine}

\usepackage{tikz,pgfplots}

\usepgfplotslibrary{fillbetween}
\usepgfplotslibrary{colormaps}
\usepgfplotslibrary{patchplots}

\renewcommand{\xRightarrow}[2][]{\ext@arrow 0359\Rightarrowfill@{#1}{#2}}

\usepackage{xcolor}
\definecolor{color1}{rgb}{0.1,0.498039215686275,0.9549019607843137}
\definecolor{alizarin}{rgb}{0.82, 0.1, 0.26}
\definecolor{antiquewhite}{rgb}{0.98, 0.92, 0.84}
\definecolor{azure}{rgb}{0.94, 1.0, 1.0}
\definecolor{offwhite}{rgb}{0.98, 0.97, 0.97}
\definecolor{pigment}{rgb}{0.2, 0.2, 0.6}

\usepackage{bm}

\newcommand{\Paths}{\text{Paths}}

\providecommand*{\cupdot}{%
  \mathbin{%
    \mathpalette\@cupdot{}%
  }%
}
\newcommand*{\@cupdot}[2]{%
  \ooalign{%
    $\m@th#1\cup$\cr
    \sbox0{$#1\cup$}%
    \dimen@=\ht0 %
    \sbox0{$\m@th#1\cdot$}%
    \advance\dimen@ by -\ht0 %
    \dimen@=.5\dimen@
    \hidewidth\raise\dimen@\box0\hidewidth
  }%
}

\newcommand{\pdist}{\Delta_{\Theta}(S)}

\usepackage{colortbl}

\makeatletter
\RequirePackage[bookmarks,unicode,colorlinks=true]{hyperref}%
   \def\@citecolor{blue}%
   \def\@urlcolor{blue}%
   \def\@linkcolor{blue}%

\def\orcidID#1{\smash{\href{http://orcid.org/#1}{\protect\raisebox{-1.25pt}{\protect\includegraphics{ORCID_Color.eps}}}}}
\makeatother

\usepackage[framemethod=TikZ]{mdframed}

\newmdtheoremenv[
  innerleftmargin=5pt,
  innerrightmargin=5pt,
  innertopmargin=5pt,
  innerbottommargin=5pt,
  skipabove=10pt,
  skipbelow=10pt
]{myproblem}{Problem}

\usepackage{bbding}

\usepackage{color}

\usepackage{cite}

\begin{document}
\title{Robust Parameter Learning for Uncertain MDPs}
%
%
\author{Yannik Schnitzer, Alessandro Abate, David Parker}
\authorrunning{Y. Schnitzer et al.}
\institute{University of Oxford, Oxford, United Kingdom} 
%
\maketitle              

\begin{abstract}
Learning-based approaches to verifying unknown Markov decision processes (MDPs) often employ uncertain MDPs. These models use, for example, confidence intervals to capture transition uncertainty and allow synthesis of policies that are robust to this uncertainty.
However, this approach typically quantifies uncertainty independently for individual transition probabilities, ignoring dependencies due to shared latent quantities. 
We propose to learn such models using parametric MDPs (pMDPs), where transition probabilities are expressions over a set of parameters.
We project statistical uncertainty from empirical transition frequencies onto the pMDP's
parameter space, yielding a probably approximately correct (PAC) uncertainty model for the underlying MDP
that respects the algebraic dependencies between transitions.
The resulting models are algorithmically challenging to solve,
so we propose~a~hierarchy~of~sound~polytopic outer approximations of the induced confidence set.
We implement and evaluate our approach, demonstrating substantially tighter uncertainty estimates than classical interval-based~uncertain~MDP~learning~techniques. 
\end{abstract}
\section{Introduction}
\label{sec:intro}

Markov decision processes (MDPs) are the standard formalism for sequential decision-making under uncertainty. Their classical formulation assumes that all transition probabilities are known exactly, which is rarely realistic in practice. Two closely related model classes relax this assumption: \emph{uncertain MDPs} (UMDPs)~\cite{DBLP:journals/mor/WiesemannKR13,DBLP:journals/mor/Iyengar05,DBLP:journals/ior/NilimG05} and \emph{parametric MDPs} (pMDPs)~\cite{DBLP:journals/fac/LanotteMT07,DBLP:conf/nfm/HahnHZ11}. UMDPs represent \emph{epistemic} uncertainty by associating transition probabilities with \emph{uncertainty sets} of admissible values;
interval MDPs (IMDPs) are a common example.
pMDPs specify transition probabilities symbolically as expressions over parameters that range over an admissible parameter space. Both formalisms thus capture imprecise model knowledge by describing not a single precise model, but a family of models, each representing a possible true underlying system.

Parametric MDPs have found applications in model repair~\cite{DBLP:conf/tacas/BartocciGKRS11,DBLP:conf/popl/LongR16}, sensitivity analysis~\cite{DBLP:conf/cav/BadingsJMTJ23}, and quality-of-service control~\cite{DBLP:journals/tr/CalinescuGJPRT16}, predominantly via the problem of \emph{parameter synthesis}~\cite{DBLP:journals/acta/CeskaDPKB17,DBLP:conf/atva/QuatmannD0JK16,DBLP:journals/fmsd/JungesAHJKQV24,DBLP:journals/tac/CubuktepeJJKT22,DBLP:conf/tacas/Cubuktepe0JKPPT17}, which partitions the parameter space into 
two regions, corresponding to parameter valuations for which the model respectively satisfies or violates a given property.
Uncertain MDPs have emerged as a standard model for reasoning about \emph{robustness}: they support the synthesis of \emph{robust} policies, i.e., policies that are optimal under the assumption that epistemic uncertainty is resolved adversarially by the environment, selecting the worst-case admissible distributions with respect to the agent's desired objective~\cite{DBLP:journals/mor/WiesemannKR13,DBLP:journals/ior/NilimG05,DBLP:journals/mor/Iyengar05}.

UMDPs are therefore well suited to \emph{learning-based} methods for the verification of MDPs whose transition probabilities are unknown.
Typically, uncertainty sets for transition probabilities are learned from interactions with the true but unknown system, for instance as statistically sound confidence sets such as confidence intervals~\cite{DBLP:conf/icml/StrehlL05,DBLP:conf/nips/SuilenS0022,DBLP:conf/cav/AshokKW19,DBLP:conf/qestformats/MeggendorferWW25,DBLP:journals/aaai/SchnitzerAP26}. The resulting learned UMDP then contains the true system with high confidence. Consequently, any robust policy synthesised on this UMDP is guaranteed, with the same confidence, to achieve on the true system at least the robust value certified by the learned~model.

Existing approaches to UMDP learning typically treat transition probabilities across the model as independent, learning separate confidence sets for each of them~\cite{DBLP:conf/birthday/SuilenBB0025,DBLP:conf/nips/SuilenS0022,DBLP:conf/aaai/BadingsA00PS22}. In practice, however, these distributions are often coupled through shared latent quantities, such as common failure probabilities, shared reliability parameters or global environmental conditions that remain fixed 
and affect multiple transitions. Such dependencies are often structurally known and are conveniently represented by pMDPs,
even though the precise values of the underlying quantities are unknown and cannot be directly observed.

This work addresses the problem of UMDP learning under known structural parametric dependencies that are captured by a pMDP. We assume a known pMDP for which an unknown parameter instantiation induces the true underlying system. From sampled interactions with this system, we first derive statistical confidence sets for the individual transition probabilities and then exploit the parametric structure of the pMDP to project these sets into the parameter space. This yields a confidence region in parameter space that contains the true underlying parameter instantiation with high confidence.

\vspace{5pt}
\noindent \textbf{Example.}
We illustrate this idea using an example: a Mars-rover navigation domain, inspired by the semi-autonomous vehicle environment in~\cite{DBLP:conf/tacas/Junges0DTK16,DBLP:journals/jfr/StucklerSSTB16}. A controller communicates with the rover via two lossy satellite channels while the rover attempts to collect samples and return home. The probability of successful communication depends on the rover's position through known coefficients, but on unknown channel reliabilities.
At each position, it is given by
$
v^\top \theta \;=\; v_1\theta_1 + v_2\theta_2,
$
where $v=(v_1,v_2)$ is position-dependent and $\theta=(\theta_1,\theta_2)$ are the unknown reliabilities. The domain is therefore naturally~modelled~as~a~pMDP~as~shown~in~Figure~\ref{fig:pipesub1}.

\begin{figure}[t]
\centering
\makebox[\textwidth][c]{%
\begin{subfigure}[t]{0.39\textwidth}
  \vspace{0pt}
  \centering
  \resizebox{\linewidth}{!}{%
  \begin{tikzpicture}[on grid, auto, very thick, >=latex, scale=0.60, transform shape]
    \tikzstyle{state}=[circle, draw, minimum size=1.2cm, align=center,
      fill=lightred, fill opacity=0.8, text opacity=1]
    \tikzstyle{split}=[draw=black, circle, fill=BrickRed, fill opacity=0.5, inner sep=3.0pt]

    \node[state] (S0) at (0,0) {\Large$\mathbf{s_0}$};
    \node[state] (S1) at (4,0) {\Large$\mathbf{s_1}$};
    \node[state] (S2) at (0,-4) {\Large$\mathbf{s_2}$};
    \node[state] (S3) at (4,-4) {\Large$\mathbf{s_3}$};

    \coordinate (S1a2) at ($(S1.east)+(1.4,0)$);
    \draw[path fading=east] (S1.east) -- (S1a2);
    \node[above, font=\large] at ($(S1.east)+(0.55,0)$) {$a$};
    \node[right] at (S1a2) {$\cdots$};

    \node[split] (Da) at ($(S0.east)+(1.0,0)$) {};
    \draw[->] (S0) -- node[above, font=\large,yshift=2pt] {$a$} (Da);
    \draw[->] (Da) -- node[above, font=\large] {$u^\top\theta$} (S1);
    \draw[->] (Da) to[bend right=12] node[right,xshift=2pt, font=\large] {$1-u^\top\theta$} (S2);

    \node[split] (Db0) at ($(S0.south)+(0,-1.0)$) {};
    \draw[->] (S0) -- node[left, font=\large,xshift=-2pt] {$b$} (Db0);
    \draw[->] (Db0) -- node[left, font=\large] {$u^\top\theta$} (S2);
    \draw[->] (Db0) to[bend left=65, looseness=1.15] node[left, font=\large,xshift=-2pt] {$1-u^\top\theta$} (S0);

    \node[split] (Db) at ($(S1.south)+(0,-1.0)$) {};
    \draw[->] (S1) -- node[right, font=\large,xshift=2pt] {$b$} (Db);
    \draw[->] (Db) -- node[right, font=\large] {$1-v^\top\theta$} (S3);
    \draw[->] (Db) to[bend left=12] node[above,xshift=-8pt,yshift=5pt, font=\large] {$v^\top\theta$} (S2);

    \node[split] (Da2) at ($(S2.east)+(1.0,0)$) {};
    \draw[->] (S2) -- node[below, font=\large,yshift=-2pt] {$a$} (Da2);
    \draw[->] (Da2) -- node[below, font=\large,yshift=-2pt] {$w^\top\theta$} (S3);
    \draw[->] (Da2) to[bend left=60, looseness=1.15] node[below, font=\large,yshift=-4pt] {$1-w^\top\theta$} (S2);

    \coordinate (S3aR) at ($(S3.east)+(1.4,0)$);
    \draw[path fading=east] (S3.east) -- (S3aR);
    \node[above, font=\large] at ($(S3.east)+(0.55,0)$) {$a$};
    \node[right] at (S3aR) {$\cdots$};

    \coordinate (S3bD) at ($(S3.south)+(0,-1.4)$);
    \draw[path fading=south] (S3.south) -- (S3bD);
    \node[right, font=\large] at ($(S3.south)+(0,-0.55)$) {$b$};
    \node[below] at (S3bD) {$\vdots$};

    \coordinate (S2bD) at ($(S2.south)+(0,-1.4)$);
    \draw[path fading=south] (S2.south) -- (S2bD);
    \node[left, font=\large] at ($(S2.south)+(0,-0.55)$) {$b$};
    \node[below] at (S2bD) {$\vdots$};
  \end{tikzpicture}%
  }
  \caption{Parametric MDP}
  \label{fig:pipesub1}
\end{subfigure}
\hspace{-0.03\textwidth}%
\begin{subfigure}[t]{0.39\textwidth}
  \vspace{0pt}
  \centering
  \resizebox{\linewidth}{!}{%
  \begin{tikzpicture}[on grid, auto, very thick, >=latex, scale=0.60, transform shape, xshift=-5pt]
    \tikzstyle{state}=[circle, draw, minimum size=1.2cm, align=center,
      fill=lightred, fill opacity=0.8, text opacity=1]
    \tikzstyle{split}=[draw=black, circle, fill=BrickRed, fill opacity=0.5, inner sep=3.0pt]

    \node[state] (S0) at (0,0) {\Large$\mathbf{s_0}$};
    \node[state] (S1) at (4,0) {\Large$\mathbf{s_1}$};
    \node[state] (S2) at (0,-4) {\Large$\mathbf{s_2}$};
    \node[state] (S3) at (4,-4) {\Large$\mathbf{s_3}$};

    \coordinate (S1a2) at ($(S1.east)+(1.4,0)$);
    \draw[path fading=east] (S1.east) -- (S1a2);
    \node[above, font=\large] at ($(S1.east)+(0.55,0)$) {$a$};
    \node[right] at (S1a2) {$\cdots$};

    \node[split] (Da) at ($(S0.east)+(1.0,0)$) {};
    \draw[->] (S0) -- node[above, font=\large,yshift=2pt] {$a$} (Da);
    \draw[->] (Da) -- node[above, font=\large] {$[0.3,0.5]$} (S1);
    \draw[->] (Da) to[bend right=12] node[right,xshift=2pt, font=\large] {$[0.5,0.7]$} (S2);

    \node[split] (Db0) at ($(S0.south)+(0,-1.0)$) {};
    \draw[->] (S0) -- node[left, font=\large,xshift=-2pt] {$b$} (Db0);
    \draw[->] (Db0) -- node[left, font=\large] {$[0.2,0.4]$} (S2);
    \draw[->] (Db0) to[bend left=65, looseness=1.15] node[left, font=\large,xshift=-2pt] {$[0.6,0.8]$} (S0);

    \node[split] (Db) at ($(S1.south)+(0,-1.0)$) {};
    \draw[->] (S1) -- node[right, font=\large,xshift=2pt] {$b$} (Db);
    \draw[->] (Db) -- node[right, font=\large] {$[0.5,0.8]$} (S3);
    \draw[->] (Db) to[bend left=12] node[above,xshift=-8pt,yshift=5pt, font=\large] {$[0.2,0.5]$} (S2);

    \node[split] (Da2) at ($(S2.east)+(1.0,0)$) {};
    \draw[->] (S2) -- node[below, font=\large,yshift=-2pt] {$a$} (Da2);
    \draw[->] (Da2) -- node[below, font=\large,yshift=-2pt] {$[0.3,0.6]$} (S3);
    \draw[->] (Da2) to[bend left=60, looseness=1.15] node[below, font=\large,yshift=-4pt] {$[0.4,0.7]$} (S2);

    \coordinate (S3aR) at ($(S3.east)+(1.4,0)$);
    \draw[path fading=east] (S3.east) -- (S3aR);
    \node[above, font=\large] at ($(S3.east)+(0.55,0)$) {$a$};
    \node[right] at (S3aR) {$\cdots$};

    \coordinate (S3bD) at ($(S3.south)+(0,-1.4)$);
    \draw[path fading=south] (S3.south) -- (S3bD);
    \node[right, font=\large] at ($(S3.south)+(0,-0.55)$) {$b$};
    \node[below] at (S3bD) {$\vdots$};

    \coordinate (S2bD) at ($(S2.south)+(0,-1.4)$);
    \draw[path fading=south] (S2.south) -- (S2bD);
    \node[left, font=\large] at ($(S2.south)+(0,-0.55)$) {$b$};
    \node[below] at (S2bD) {$\vdots$};
  \end{tikzpicture}%
  }
\phantomsubcaption\label{fig:threepanel_polytope}
\makebox[\linewidth][c]{\hspace*{-8mm}\raisebox{-0.8mm}{(b)\hspace{0.5em}Learned IMDP}}
  \label{fig:pipesub2}
\end{subfigure}
\hspace{-0.01\textwidth}
\begin{subfigure}[t]{0.25\textwidth}
  \vspace{0pt}
  \centering
  \resizebox{\linewidth}{!}{%
  \begin{tikzpicture}[scale=0.85, xshift=4mm, yshift=16mm]
    \shade[shading=axis,
           bottom color=white, top color=lightred!30,
           shading angle=90,
           opacity=0.75] (0,0) rectangle (3.2,3.2);
    \shade[shading=axis,
           left color=white, right color=lightred!30,
           shading angle=0,
           opacity=0.75] (0,0) rectangle (3.2,3.2);

    \draw[-{Latex[length=4mm,width=3mm]}, line width=1.6pt] (0,0) -- (3.2,0)
      node[midway, below] {$\theta_1$};
    \draw[-{Latex[length=4mm,width=3mm]}, line width=1.6pt] (0,0) -- (0,3.2)
      node[midway, left] {$\theta_2$};

    \tikzfading[name=bothends,
      left color=transparent!100,
      right color=transparent!100,
      middle color=transparent!0]

    \draw[path fading=bothends, fading angle=90, opacity=0.55, line width=1.4pt] (0.95,0.30) -- (0.95,2.80);
    \draw[path fading=bothends, fading angle=90, opacity=0.55, line width=1.4pt] (2.10,0.30) -- (2.10,2.80);

    \draw[path fading=bothends, opacity=0.55, line width=1.4pt] (0.15,1.00) -- (2.25,3.10);
    \draw[path fading=bothends, opacity=0.55, line width=1.4pt] (1.05,0.30) -- (2.85,2.10);

    \draw[path fading=bothends, opacity=0.55, line width=1.4pt] (0.10,1.95) -- (2.0,0.05);
    \draw[path fading=bothends, opacity=0.55, line width=1.4pt] (0.85,2.85) -- (2.95,0.75);

    \node[font=\scriptsize, anchor=west] at (2.05,2.45) {$u^\top\theta$};
    \node[font=\scriptsize, anchor=west] at (2.10,3.20) {$v^\top\theta$};
    \node[font=\scriptsize, anchor=west] at (2.25,0.50) {$w^\top\theta$};

    \coordinate (P1) at (0.95,1.10);
    \coordinate (P2) at (0.95,1.80);
    \coordinate (P3) at (1.425,2.275);
    \coordinate (P4) at (2.10,1.60);
    \coordinate (P5) at (2.10,1.35);
    \coordinate (P6) at (1.40,0.65);

    \fill[BrickRed, opacity=0.5] (P1)--(P2)--(P3)--(P4)--(P5)--(P6)--cycle;
    \draw[BrickRed!70!black, line width=0.9pt] (P1)--(P2)--(P3)--(P4)--(P5)--(P6)--cycle;
    \node at (1.50,1.50) {\large $\mathcal{U}$};
  \end{tikzpicture}%
  }
\phantomsubcaption\label{fig:threepanel_polytope}
\makebox[\linewidth][c]{\hspace*{-6mm}\raisebox{-1.8mm}{(c)\hspace{0.5em}Confidence Region}}
  \label{fig:threepanel_polytope}
\end{subfigure}%
}
\caption{Construction of the parameter confidence region $\mathcal{U}$ from a learned IMDP.}
\label{fig:overview_threepanel_pipeline}
\vspace{-6pt}
\end{figure}
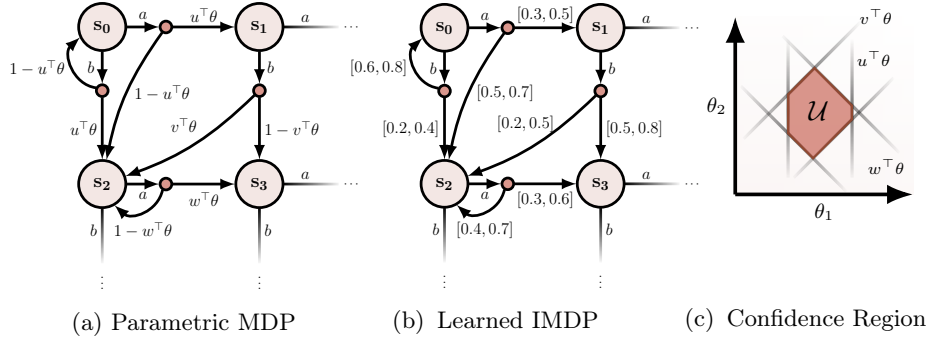

From sampled transitions of the rover, the standard approach learns an IMDP by assigning a confidence interval to each transition probability; see Figure~\ref{fig:pipesub2}. This yields a robust model with formal confidence guarantees on inclusion of the true unknown system, but treats transition probabilities as independent. In our example, however, they are coupled through the shared parameter instantiation~$\theta$.

Our approach exploits this structure by projecting the learned transition-wise intervals through the parametric structure of the pMDP into the~parameter~space. Each learned transition confidence interval $[l,u]$ induces constraints of the form
\[
l \leq v_1\theta_1 + v_2\theta_2 \leq u.
\]
Collecting these constraints yields a parameter confidence region $\mathcal{U}$, as shown in Figure~\ref{fig:threepanel_polytope}. With high confidence, the true parameter instantiation lies in $\mathcal{U}$.

The confidence region $\mathcal{U}$, together with the pMDP,
induces a UMDP $U$ that contains the true MDP with high confidence.
This is tighter than a UMDP constructed from the confidence sets for individual transitions
(e.g., the IMDP in Fig.~\ref{fig:pipesub2})
since it rules out transition functions that are inconsistent with any of the constraints induced by the parametric expressions and the learned intervals, leading to less conservative robust policies and~tighter~guarantees~from~the~same~data.

However, because of its (non-rectangular) structure, the UMDP $U$ is computationally very challenging to solve. 
We tackle this problem by proposing a hierarchy of relaxations of the induced uncertainty set,
yielding UMDPs that can be solved more efficiently to generate robust policies.
All relaxations are sound and thus produce \emph{probably approximately correct} (PAC)~\cite{DBLP:conf/stoc/Valiant84} performance guarantees. 
This hierarchy illustrates a clear trade-off between the tightness and efficiency of the relaxations.
We implement our methods and evaluate them on a range of established benchmarks.
The results demonstrate substantially tighter uncertainty estimates
than classical interval-based uncertain MDP~learning~techniques.

\section{Preliminaries}
\label{sec:prelim}

For a finite set $X$, let $\Delta(X) = \{\mu \colon X \to [0,1] \mid \sum_{x \in X} \mu(x) = 1\}$ denote the set of probability distributions over $X$. 

\begin{definition}[Markov decision process]
A (finite) MDP is a tuple
$
M=(S,A,s_0,P),
$
where $S$ is a finite set of states, $A$ is a finite set of actions, $s_0\in S$ is an initial state, and
$P:S\times A \to \Delta(S)$ is a transition kernel. We write $P(s,a)(s')$ (or $P(s,a,s')$) for the
probability to move from state $s$ to state $s'$ under action $a$.
\end{definition}

A \emph{policy} (or \emph{strategy}) is a function
$\pi \colon S \to A$ that selects an action $\pi(s)$ in each state $s$.
An (infinite) \emph{path} is a sequence $\rho = s^0 a^0 s^1 a^1 \dots \in (S\times A)^\omega$ such that
$P(s^t,a^t)(s^{t+1})>0$ for all $t\ge 0$.
We write $\Paths(s)$ for the set of paths starting in $s$ 
and $\Pi$ for the set of policies.
An \emph{objective} is a measurable mapping $R \colon \Paths(s) \to \mathbb{R}$ assigning a return to each path.
Typical objectives are based on satisfaction of temporal-logic properties~\cite{BdA95}
or accumulation of rewards~\cite{DBLP:books/wi/Puterman94}. 
A policy $\pi\in\Pi$ executed in an MDP $M$ induces a probability measure over $\Paths(s)$~\cite{DBLP:books/wi/Puterman94}. We define the \emph{value} of $\pi$ in~$s$~as~$
V^\pi_M(s) \coloneqq \mathbb{E}^{\pi}_{M,s}[R]\, .
$

\subsection{Parametric and Uncertain MDPs}

Let $\Theta=\{\theta_1,\dots,\theta_\ell\}$ be an ordered set of parameters.
A \emph{parameter instantiation} is a vector $\mathbf{u} \in \mathbb{R}^\ell$. We assume that all parameters are bounded, i.e., $l_i \leq \mathbf{u}_i \leq u_i$ for some $l_i,u_i \in \mathbb{R}$, and refer to $\mathcal{D} \coloneqq \prod_{i=1}^{\ell} [l_i,u_i]$ as the \emph{parameter space}.
Let $\mathbb{Q}[\Theta]$ denote the ring of polynomials over $\Theta$ with rational coefficients.
Each $f\in\mathbb{Q}[\Theta]$ induces a function $f:\mathbb{R}^\ell\to\mathbb{R}$, where $f[\mathbf{u}]$ denotes the polynomial $f$ evaluated by substituting every occurrence of $\theta_i$ with its instantiation $\mathbf{u}_i$. We denote by \[\Delta_{\Theta}(X) = \{f \colon X \to \mathbb{Q}[\Theta] \mid \forall \uu \in \mathcal{D}. \; f(x)[\uu] \in [0,1], \sum_{x\in X} f(x)[\uu] = 1\}\] the set of all \emph{parametric probability distributions} over a finite set $X$, i.e., the set of all parametric functions that induce a valid probability distribution over $X$ for any parameter instantiation $\uu \in \mathcal{D}$.

Parametric Markov decision processes extend standard MDPs by allowing transition probabilities to be specified as parametric probability distributions, where fixing a parameter instantiation yields an ordinary MDP.
\begin{definition}[Parametric MDP]
\label{def:pmdp}
A \emph{parametric MDP} (pMDP) is a tuple
$M_\Theta=(S,A,s_0,\Theta,P_\Theta)$, where $S,A,s_0$ are as for MDPs, $\Theta$ is a parameter set, and
$
P_\Theta: S\times A \to \pdist
$
is a \emph{parametric} transition probability function.
For a parameter instantiation $\uu\in\mathbb{R}^l$, we obtain the instantiated MDP $M[\uu]=(S,A,s_0,P[\uu])$ with
$
P[\uu](s,a)(s') := P_{\Theta}(s,a)(s')[\uu]$, for all $s,s'\in S,\ a\in A$.
\end{definition}

In this work, we assume there is a true but unknown MDP $M$, obtained by instantiating a \emph{known} pMDP $M_\Theta$ with an \emph{unknown} parameter instantiation $\uu$.
The objective is to learn, from sampled interactions, an approximation of the true dynamics with quantified uncertainty, and to synthesise policies that are robust to the residual uncertainty. 
To represent partial knowledge, 
we adopt the framework of UMDPs.
A UMDP is an MDP in which transition probabilities are not fixed, but are only required to lie in~specified~\emph{uncertainty sets}.

\begin{definition}[Uncertain MDP]
\label{def:umdp}
A UMDP is a tuple
$
U=(S,A,s_0,\mathcal{P}),
$
where $S,A,s_0$ are as for MDPs and
$
\mathcal{P} \subseteq (\Delta(S))^{S \times A}
$
is an \emph{uncertainty set} of admissible transition probability functions $P \in \mathcal{P}$. The UMDP is \emph{(s,a)-rectangular} if $\mathcal{P}$
factorises as the product of local uncertainty sets, i.e.,
$
\mathcal{P}
\;=\;
\prod_{(s,a)\in S\times A}\mathcal{P}(s,a),
$
where $\mathcal{P}(s,a) \subseteq \Delta(S)$.
Otherwise, the uncertainty is \emph{coupled} (non-rectangular).

\end{definition}

There is a close relationship between pMDPs and UMDPs.
Given a pMDP and a set $\mathcal{U}$ of admissible parameter instantiations, we obtain a UMDP by letting the uncertainty set contain exactly the successor distributions induced by instantiations $\uu \in \mathcal{U}$.
Conversely, a UMDP can be represented as a pMDP together with a set of admissible parameter instantiations inducing precisely the uncertainty set.
We formalise this correspondence for completeness in Appendix~\ref{app:pmdpumdp}.

In this work, the known pMDP serves as a structural model, capturing dependencies across transitions.
From sampled interactions with the unknown system, we infer a set of plausible parameter instantiations $\mathcal{U}$ with quantified confidence and, equivalently, a UMDP whose uncertainty sets contain the transition distributions consistent with $\mathcal{U}$.
This UMDP is then used for robust policy synthesis under the remaining uncertainty about the true dynamics of $M$.

\vspace{5pt}
\noindent \textbf{Interval MDPs.}
We recall a standard class of UMDPs that we use throughout. 
An \emph{interval MDP} (IMDP) specifies, for each $(s,a,s')\in S\times A\times S$, componentwise lower and upper bounds
$\ell_{(s,a,s')},u_{(s,a,s')}\in[0,1]$.
The admissible successor distributions from $(s,a)$ then are those whose components respect these bounds:
\[
\mathcal{P}(s,a)
\;=\;
\Bigl\{\mu \in \Delta(S)\ \Big|\ \forall s'\in S:\ \ell_{(s,a,s')} \le \mu(s') \le u_{(s,a,s')} \Bigr\}.
\]

\subsection{Solving Uncertain MDPs}
\label{sec:solving_umdp}

For a UMDP, a \emph{robust} policy is one that optimises worst-case performance over all admissible transition kernels.
Following the standard robust MDP interpretation~\cite{DBLP:journals/ior/NilimG05,DBLP:journals/mor/Iyengar05,DBLP:journals/mor/WiesemannKR13},
we view uncertainty as being resolved adversarially by an \emph{environment} (or \emph{nature}) that selects admissible transition probabilities.
Formally, an agent policy $\pi$ and an admissible transition kernel $P \in \mathcal{P}$ induce a probability measure over paths~\cite{DBLP:conf/cdc/WolffTM12}. We denote the corresponding expectation by $\mathbb{E}^{\pi}_{P,s}[R]$ and define the value of state $s$ under $(\pi,P)$ as
$
V^{\pi}_{P}(s) \coloneqq \mathbb{E}^{\pi}_{P,s}[R].
$
The \emph{robust value} is the maximal expected return under the worst-case nature:
\begin{equation}
V^*_U(s)
\;\coloneqq\;
\sup_{\pi \in \Pi}\ \inf_{\raisebox{-0.5ex}{$\scriptstyle P \in \mathcal{P}$}}\ V^{\pi}_P(s),
\label{eq:robust_value_def}
\end{equation}
and a \emph{robust-optimal policy} is any
$
\pi^*\in \text{argsup}_{\pi}\inf_{P} V^{\pi}_P(s).
$
In particular, $\pi^*$ guarantees to achieve at least $V^*_U(s)$ for any admissible transition kernel $P\in\mathcal{P}$.

Computing~\eqref{eq:robust_value_def} can be challenging in general, in particular for coupled uncertainty sets~\cite{DBLP:journals/mor/WiesemannKR13}.
For \emph{rectangular} uncertainty (i.e.\ $\mathcal{P}=\prod_{(s,a)}\mathcal{P}(s,a)$), the inner minimisation decomposes into local optimisations at each state--action pair, and the robust value satisfies a Bellman-type optimality equation.
For example, for reachability with target set $T\subseteq S$, the robust Bellman equations take the form
\begin{equation}
V^*_U(s)
=
\max_{a\in A}\ \underbrace{\min_{\mu \in \mathcal{P}(s,a)}\ \sum_{s'\in S} \mu(s')\, V^*_U(s')}_{\text{Inner Optimisation}},
\quad s\in S\setminus T,
\label{eq:robust_bellman_rect}
\end{equation}
together with boundary conditions $V^*_U(s)=1$ for all $s\in T$.
Analogous robust Bellman equations can be given, e.g., for expected or discounted reward objectives~\cite{DBLP:journals/ior/NilimG05,DBLP:journals/mor/Iyengar05}.
These equations can be solved with \emph{robust value iteration}, i.e., by repeatedly applying the robust Bellman operator until convergence~\cite{DBLP:journals/ior/NilimG05,DBLP:journals/mor/Iyengar05}. 

\subsection{Learning Uncertain MDPs}
\label{sec:learn-umdp}

We review the basics of learning an UMDP from data generated by an unknown MDP $M$.
We assume access to a dataset of transition samples
$
C=\{(s_i,a_i,s_i')\}_{i=1}^n,
$
obtained either (i) by collecting trajectories in the MDP, i.e., episodic sampling with restarts from the initial state, or continuing interaction, or
(ii) via a \emph{generative model} that allows direct sampling of next states for state--action pairs~\cite{DBLP:conf/icml/StrehlL05,DBLP:conf/birthday/SuilenBB0025}.

From $C$ we extract the empirical counts
$\#(s,a)$ and $\#(s,a,s')$ for all $s,s'\in S,\ a\in A$,
where $\#(s,a)$ is the number of occurrences of $(s,a)$ and $\#(s,a,s')$ the number of observed transitions from $s$ to $s'$ under action $a$.
Whenever $\#(s,a)>0$, the empirical estimate of the transition probability is
$
\tilde P(s,a,s') \ :=\ \frac{\#(s,a,s')}{\#(s,a)} .
$

A standard UMDP learning technique is to construct an IMDP that over-approximates the unknown transition probabilities by confidence intervals. For each $(s,a,s')$, we view $\#(s,a,s')$ as a binomial observation with $\#(s,a)$ trials and success probability $P(s,a,s')$, and lift the empirical estimate $\tilde P(s,a,s')$ to a confidence interval
$
I_\gamma(s,a,s') \ :=\ [\,\ell(s,a,s'),\,u(s,a,s')\,],
$
such that
$
\Pr\bigl[P(s,a,s')\in I_\gamma(s,a,s')\bigr]\ \ge\ 1-\gamma,
$
for $\gamma \in (0,1)$. Suitable choices for $I_\gamma$ include standard binomial confidence intervals such as the Clopper--Pearson interval~\cite{ClopperPearson1934,DBLP:conf/qestformats/MeggendorferWW25}. If $\#(s,a)=0$, no information is available for that state--action pair, and we therefore use the trivial interval $I_\gamma(s,a,s')=[0,1]$ for all $s'\in S$. 

For the objectives considered in this paper, such as reachability probabilities, intervals containing both zero and non-zero values can be handled directly by the robust value iteration in Eq.~\eqref{eq:robust_bellman_rect}. For more general temporal-logic objectives, however, satisfaction may depend on the graph structure of the UMDP. In that case, intervals containing both zero and non-zero values allow the adversarial nature to alter the graph structure itself, which complicates robust synthesis. This issue is not specific to our approach, but inherent to temporal-logic reasoning over UMDPs in general. Existing techniques therefore often assume \emph{constant support}, i.e., that the set of successors is known a priori~\cite{DBLP:conf/cdc/WolffTM12,DBLP:conf/nips/SuilenS0022,DBLP:conf/birthday/SuilenBB0025}. Under this assumption, one may replace any lower bound $0$ in the learned intervals by a sufficiently small $\varepsilon>0$ in order to preserve the known graph structure. 

To obtain an overall confidence level $1-\delta$, with $\delta \in (0,1)$, we distribute the confidence budget across all transition probabilities and set
$\gamma := \delta / |S\times A \times S|$.
Hence, with probability at least $1-\delta$, the unknown MDP is contained in the learned IMDP
(e.g.,~\cite{DBLP:conf/nips/SuilenS0022,DBLP:journals/aaai/SchnitzerAP26}).
Consequently, the value of a policy on the learned IMDP 
yields a ($1{-}\delta$)-valid lower bound on the performance of the policy in $M$.
Other common constructions of uncertainty sets $\mathcal{P}(s,a)$ with formal confidence guarantees include, e.g., $L_1$-norm balls based on the Weissman inequality~\cite{DBLP:conf/icml/StrehlL05,WeissmanEtAl2003}.

\subsection{Parameter Tying}
\label{sec:parametertying}

An established method for improving IMDP learning in the presence of known parametric dependencies is \emph{parameter tying}~\cite{DBLP:conf/tacas/SchnitzerAP25,DBLP:conf/qest/PolgreenWHA16,DBLP:journals/corr/abs-2507-15532}. Consider two transitions $(s,a,s')$ and $(t,b,t')$ in $M_\Theta$ with $s \neq t$ and $P_\Theta(s,a,s') = P_\Theta(t,b,t')$, i.e., both transition probabilities are described by the same expression.
Then the samples in $C$ obtained for these transitions correspond to the same Bernoulli experiment and can be \emph{pooled} to derive a single interval, which is then assigned to both transitions.
Moreover, the number of components to be learned over which we distribute the overall confidence budget $\delta$ reduces to the number of distinct parametric expressions in $M_\Theta$. 
Details on parameter tying are provided in Appendix~\ref{app:parametertying}.

Parameter tying applies only when parametric expressions are \emph{exactly} equivalent. 
Our approach instead aggregates information from the parametric structure and the collected samples by projecting the learned constraints through the algebraic structure of $M_\Theta$ into the parameter space. This jointly accounts for observations across all transitions and can tighten uncertainty even for transitions whose expressions have not been observed in any sample in $C$.
We use parameter tying as a pre-processing step: we pool the observed counts for equivalent parametric expressions and then apply our technique to the pooled data. Let
\begin{equation}
   \Lambda \;=\; \{\,P_{\Theta}(s,a,s') \mid s,s' \in S,\ a \in A\,\} 
\end{equation}
be the set of distinct parametric expressions occurring in $M_{\Theta}$. We distribute the confidence budget only across those expressions that are not constant polynomials:
$
\Lambda_{\mathsf{unk}} \;:=\; \{\, f \in \Lambda \mid f \notin \mathbb{Q}\,\}.
$
For each $f \in \Lambda_{\mathsf{unk}}$, we compute from the pooled counts a confidence interval $[l_f,u_f]$ as per Section~\ref{sec:learn-umdp} such that, under the true but unknown parameter instantiation $\uu$,
$
\Pr\bigl[\,l_f \le f[\uu] \le u_f\,\bigr] \;\ge\; 1 - \frac{\delta}{|\Lambda_{\mathsf{unk}}|}.
$
For constant expressions $f\in\Lambda\setminus\Lambda_{\mathsf{unk}}$, we set $l_f=u_f=f$.
Substituting each occurrence of $f$ in $M_\Theta$ by the corresponding interval $[l_f,u_f]$ yields an IMDP that contains the true (unknown) MDP $M$ with probability at least $1-\delta$.

\section{Robust Parametric UMDP Learning}
\label{sec:learning}

\noindent
IMDP learning 
inherently assumes that all transition probabilities are \emph{independent} and learned as individual confidence intervals. 
We now introduce our approach which leverages dependencies specified by an underlying pMDP model. 
From the observed transitions, we infer a set of plausible instantiations $\mathcal{U}$ with quantified confidence, exploiting the dependencies captured by shared~parameters. 

Throughout, we consider a fixed pMDP $M_\Theta$ and an unknown MDP $M$ induced by $M_\Theta$ under an unknown parameter instantiation $\uu$. We further fix a data set of i.i.d.\ transition samples $C$ obtained from $M$, as described in Section~\ref{sec:learn-umdp}. We focus initially on the case where the per-transition uncertainty sets are intervals, i.e., we learn an IMDP approximation of $M$ and then project it into the parameter space. The same ideas extend to other classes of uncertainty sets, such as $L_1$-balls that preserve per-state coupling~\cite{DBLP:conf/icml/StrehlL05}, as we shall discuss in Section~\ref{sec:otherclasses}.

Given the learned intervals $[l_f,u_f]$ for $f \in \Lambda$, we define the \emph{uncertainty region}
\begin{equation}
\label{eqn:constraintsys}
\mathcal{U}
\;=\;
\bigl\{\,\vv \in \mathcal{D} \ \big|\ \forall f \in \Lambda:\ l_f \leq f[\vv] \leq u_f \,\bigr\}.
\end{equation}
The region $\mathcal{U}$ contains exactly those parameter instantiations that satisfy all learned interval constraints simultaneously. Together with the pMDP $M_\Theta$, it induces a UMDP
$
U = (S, A, s_0, \mathcal{P}_{\mathcal{U}}),
$
where the uncertainty set consists of all transition kernels obtained by instantiating $M_\Theta$ with a parameter vector in $\mathcal{U}$:
\[
\mathcal{P}_{\mathcal{U}}
\;\coloneqq\;
\bigl\{\, P[\vv] \in (\Delta(S))^{S \times A} \ \big|\ \vv \in \mathcal{U} \,\bigr\}.
\]

\begin{theorem}[Soundness of $\mathcal{U}$]
\label{thm:soundness}
With probability at least $1-\delta$, the true parameter instantiation $\uu$ satisfies $\uu \in \mathcal{U}$. Equivalently,
\[
\Pr\bigl[\;P[\uu] \in \mathcal{P}_{\mathcal{U}}\;\bigr] \;\ge\; 1-\delta.
\]
Consequently, for any policy $\pi \in \Pi$ and state $s \in S$, with probability at least $1-\delta$ the robust value under the induced UMDP lower-bounds the true value in $M$, i.e.,
\[
\Pr\bigl[\;V^\pi_{U}(s) \leq V^{\pi}_{M}(s)\;\bigr] \;\ge\; 1-\delta.\tag*{\qed}
\]
\end{theorem}
The induced UMDP $U$ also yields a robust policy $\pi^*$.
By Theorem~\ref{thm:soundness}, $\pi^*$ achieves, with probability at least $1-\delta$, at least its robust value $V^\pi_{U}(s)$ on the unknown MDP $M[\uu]$.
However, computing the robust value and policy from $U$ remains very challenging.
We summarise the key issues below and then describe in more detail how we address them in the remainder of this section. 

\begin{enumerate}
    \item \textbf{Non-rectangularity and tractability (Sec.~\ref{sec:rectrelax}).} 
    The uncertainty set $\mathcal{P}_{\mathcal{U}}$ for $U$ is in general \emph{non-rectangular}: shared parameters couple uncertainty across states and actions.
    As a consequence, the robust optimisation problem in~\eqref{eq:robust_value_def} becomes very difficult and robust policy synthesis is NP-hard in general~\cite{DBLP:journals/mor/WiesemannKR13}. We therefore introduce a family of \emph{rectangular} relaxations that over-approximate $\mathcal{P}_{\mathcal{U}}$. These relaxations preserve the soundness of Theorem~\ref{thm:soundness},  while enabling efficient synthesis. For each relaxation, we analyse (i) the cost of constructing it, (ii) the resulting complexity of robust policy synthesis, and (iii) its tightness. Further, we establish an inclusion hierarchy among the uncertainty sets, making the tightness--efficiency~trade-off~explicit.

    \item \textbf{Polynomial constraints and linearisation (Sec.~\ref{sec:mccormick}).}
    The uncertainty set $\mathcal{P}_{\mathcal{U}}$ is in general defined by polynomial constraints, since transition probabilities may be polynomial functions of the parameters. As a result, even after rectangularisation, the inner optimisation in~\eqref{eq:robust_bellman_rect} can remain challenging, as it requires optimisation under polynomial constraints. We therefore leverage \emph{linear relaxations} that transform these polynomial constraints into linear ones. Combined with rectangular relaxations, we show that this yields tractable inner problems,  while retaining tight uncertainty sets.

    \item \textbf{Possible emptiness of $\mathcal{P}_{\mathcal{U}}$ (Sec.~\ref{sec:empty-R}).}
    Finally, we show that, unlike in standard IMDP learning, the induced uncertainty set $\mathcal{P}_{\mathcal{U}}$ may be empty. We give a statistical interpretation of this edge case, derive an upper bound on the probability that it occurs, and provide a principled fallback mechanism.
    
\end{enumerate}

\vspace*{-1em}
\subsection{Rectangular Relaxations}
\label{sec:rectrelax}

Robust policy synthesis over $\mathcal{P}_{\mathcal{U}}$ is intractable in general, since $\mathcal{P}_{\mathcal{U}}$ may be non-rectangular. We therefore introduce a range of rectangular relaxations, i.e., uncertainty sets that over-approximate $\mathcal{P}_{\mathcal{U}}$. 
These relaxations preserve the 
inclusion guarantee for the unknown MDP $M[\uu]$ while enabling tractable synthesis. For each relaxation, we describe the additional effort required to construct it,
as well as the resulting complexity of policy synthesis, measured by the cost of the inner minimisation in~\eqref{eq:robust_bellman_rect} in each iteration of robust value iteration. We organise the relaxations into an inclusion hierarchy relative to $\mathcal{P}_{\mathcal{U}}$ and to the baseline uncertainty set $\mathcal{P}_{\mathrm{I}}$ induced solely by the intervals, as depicted in Figure~\ref{fig:tightness-diagram}.

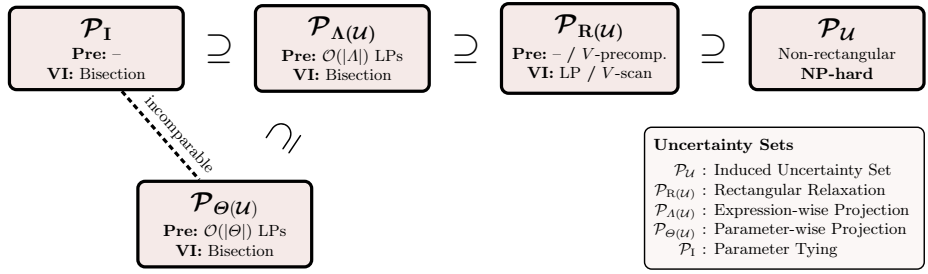
\begin{figure}[t]
\centering
\resizebox{\textwidth}{!}{%
\begin{tikzpicture}[node distance=1.35cm, >={Stealth[length=3.5mm,width=2.8mm]}, line width=1.8pt]
  \tikzstyle{setnode}=[draw, rounded corners, inner sep=6pt, fill=lightred!70, align=center, minimum width=3.4cm, minimum height=1.6cm]
  \tikzstyle{legendbox}=[draw, rounded corners, line width=0.8pt, inner sep=5pt, align=left, fill=lightred!28]

  \node[setnode] (imdpt) {{\Large$\boldsymbol{\mathcal{P}}_{\boldsymbol{\mathrm{I}}}$}\\[2pt]\textbf{Pre:} --\\\textbf{VI:} Bisection};
  \node[setnode, right=of imdpt] (expr) {{\Large$\boldsymbol{\mathcal{P}}_{\boldsymbol{\mathrm{\Lambda}(\mathcal{U})}}$}\\[2pt]\textbf{Pre:} $\mathcal{O}(|\Lambda|)$ LPs\\\textbf{VI:} Bisection};
  \node[setnode, right=of expr] (rect) {{\Large$\boldsymbol{\mathcal{P}}_{\boldsymbol{\mathrm{R}(\mathcal{U})}}$}\\[2pt]\textbf{Pre:} -- $\bm{/}$ $V$-precomp.\\\textbf{VI:} LP $\bm{/}$ $V$-scan};
  \node[setnode, right=of rect] (R) {{\Large$\boldsymbol{\mathcal{P}}_{\boldsymbol{\mathcal{U}}}$}\\[2pt] Non-rectangular\\\textbf{NP-hard}};

  \node[setnode, below=1.7cm of expr, xshift=-2.3cm] (box) {{\Large$\boldsymbol{\mathcal{P}}_{\boldsymbol{\Theta}(\boldsymbol{\mathcal{U}})}$}\\[2pt]\textbf{Pre:} $\mathcal{O}(|\Theta|)$ LPs\\\textbf{VI:} Bisection};

  \node at ($(imdpt)!0.5!(expr)$) {\huge$\supseteq$};
  \node at ($(expr)!0.5!(rect)$)  {\huge$\supseteq$};
  \node at ($(rect)!0.5!(R)$)     {\huge$\supseteq$};

  \path (box) -- node[midway,sloped] {\huge$\supseteq$} (expr);

  \draw[densely dashed] ([xshift=5mm]imdpt.south) -- node[above,sloped,pos=0.6]{incomparable} ([xshift=-5mm]box.north);

  \node[legendbox, anchor=north west] at ($(rect.south east)+(-0.7cm,-0.62cm)$) {%
    \begin{tabular}{@{}r@{\hspace{0.35em}}c@{\hspace{0.45em}}l@{}}
      \multicolumn{3}{@{}l@{}}{\textbf{Uncertainty Sets}}\\[3pt]
      $\mathcal{P}_{\mathcal{U}}$ & : & Induced Uncertainty Set\\
      $\mathcal{P}_{\mathrm{R}(\mathcal{U})}$ & : & Rectangular Relaxation\\
      $\mathcal{P}_{\Lambda(\mathcal{U})}$ & : & Expression-wise Projection\\
      $\mathcal{P}_{\Theta(\mathcal{U})}$ & : & Parameter-wise Projection\\
      $\mathcal{P}_{\mathrm{I}}$ & : & Parameter Tying
    \end{tabular}
  };
\end{tikzpicture}
}%
\caption{Inclusion hierarchy of the considered uncertainty sets. For each set, we report the computational~effort required to construct the relaxation (\textbf{Pre}), in addition to computing the confidence intervals and constructing the model, and the effort per inner minimisation during robust value iteration (\textbf{VI}).}
\label{fig:tightness-diagram}
\vspace{-10pt}
\end{figure}

From here on, for ease of presentation, we assume that each parametric expression $f \in \Lambda$ is linear in the parameters. This comes without loss of generality, as we will discuss in Section~\ref{sec:mccormick} how to relax polynomial constraints into linear ones while preserving soundness of the resulting uncertainty sets. Recall that we consider the case where the \emph{original} uncertainty sets derived from data are confidence intervals. As we shall argue in Section~\ref{sec:otherclasses}, our results on rectangular relaxations extend to other uncertainty classes, such as $L_1$-balls and ellipsoids. 

\vspace*{0.3em}
\noindent \textbf{Rectangularisation via local projections.}
The most natural way to relax a non-rectangular uncertainty set into a rectangular one is to allow the adversarial environment to choose its worst-case transition probabilities independently in each state--action pair. Formally, given the coupled uncertainty set $\mathcal{P}_{\mathcal{U}}$, its \emph{rectangular relaxation} $\mathcal{P}_{R(\mathcal{U})}$ is the product of its local projections:
\[
\mathcal{P}_{R(\mathcal{U})}
\;:=\;
\prod_{(s,a)\in S\times A} \mathcal{P}_{R(\mathcal{U})}(s,a),
\quad\text{where}\quad
\mathcal{P}_{R(\mathcal{U})}(s,a)
\;\coloneqq\;
\bigl\{\, P[\vv](s,a)\ \big|\ \vv \in \mathcal{U} \,\bigr\}.
\]
Rather than requiring a single instantiation $\vv\in\mathcal{U}$ to be used consistently throughout the pMDP, the environment can select an instantiation from $\mathcal{U}$ separately for each state--action pair. Essentially, this corresponds to renaming each shared parameter, thus eliminating dependencies~\cite{DBLP:conf/atva/QuatmannD0JK16,DBLP:conf/atva/HeckQSKJ25}. This is a relaxation as it grants the environment more power: it may choose a different worst-case instantiation from $\mathcal{U}$ for each state--action pair, which can only decrease~the~robust~value.

Since each parametric expression $f \in \Lambda$ is linear and constrained by interval bounds $l_f \leq f[\vv] \leq u_f$, the uncertainty region $\mathcal{U}$ is a polytope. Hence, for a state--action pair $(s,a)$, the inner minimisation in~\eqref{eq:robust_bellman_rect} reduces to a linear program:
\[
\min_{\vv \in \mathcal{U}} \ \sum_{s'\in S} P[\vv](s,a,s') \, V^*_{U}(s').
\]

\vspace*{-1em}
\noindent
or, equivalently:
\[
\begin{aligned}
\min_{\vv \in \mathcal{D}} \quad & \sum_{s'\in S} P[\vv](s,a,s') \, V^*_{U}(s') \quad
\text{s.t.}\quad l_f \leq f[\vv] \leq u_f, \; \forall f \in \Lambda.
\end{aligned}
\]
Hence, 
policy synthesis via robust value iteration requires solving a linear program for each state--action pair in every iteration. In practice, this can often still be efficient for two reasons. First, the feasible region is fixed: the constraints defining $\mathcal{U}$ depend only on the parametric expressions and the learned bounds $[l_f,u_f]$, and thus can be constructed once and reused throughout value iteration. 
Second, successive iterations induce only small changes in the objective, since the coefficients $P[\vv](s,a,s')$ are fixed and the value estimates converge. 
Using a simplex-based LP solver~\cite{DBLP:books/daglib/0095265,DBLP:books/daglib/0014645}, we can \emph{warm-start} from the vertex that was optimal in the previous iteration. Since the changes in the objective
become smaller as value iteration converges, the previously optimal vertex is often still optimal, or close to optimal, so only few steps are needed to update the solution.

Alternatively, since each inner optimisation is over the same feasible region $\mathcal{U}$, which remains fixed, we can precompute a vertex representation of $\mathcal{U}$, i.e., its set of extreme points. This is closely related to the enumeration approach used in parameter synthesis when the parameter space is a hyperrectangle~\cite{DBLP:conf/atva/QuatmannD0JK16,DBLP:conf/atva/HeckQSKJ25}. The inner minimisation then reduces from solving an LP to evaluating the objective at the vertices.
However, converting a polytope 
to vertex representation can be expensive in general, and the number of vertices may grow exponentially with the dimension of the parameter space. This approach is therefore only suitable for low-dimensional parameter spaces, as we examine in the evaluation in Section~\ref{sec:experiments}.

\vspace{10pt}
\noindent \textbf{Expression-wise projections.}
We next consider a relaxation that projects $\mathcal{U}$ back to interval bounds over the parametric expressions. 
This aggregates information from all transitions jointly through $\mathcal{U}$ and can yield substantially tighter intervals than those obtained directly
through parameter tying. The resulting model is an IMDP, whose inner optimisation in robust value iteration can be solved efficiently via bisection~\cite{DBLP:journals/ior/NilimG05,DBLP:journals/mor/Iyengar05}.

For each expression $f \in \Lambda$, we derive refined bounds $l^\Lambda_f$ and $u^\Lambda_f$ as the minimum and maximum values that $f[\vv]$ can attain over $\vv \in \mathcal{U}$. Concretely, we solve
\[
\begin{aligned}
l^\Lambda_f \;:=\; \min_{\vv \in \mathcal{U}} \ f[\vv]
\qquad\qquad
u^\Lambda_f \;:=\; \max_{\vv \in \mathcal{U}} \ f[\vv].
\end{aligned}
\]
We define the corresponding rectangular interval uncertainty set $\mathcal{P}_{\Lambda(\mathcal{U})}$ as
\[
\mathcal{P}_{\Lambda(\mathcal{U})}(s,a)
\;\coloneqq\;
\Bigl\{\mu\in\Delta(S)\ \Big|\ \forall s'\in S:\ l^\Lambda_{P_\Theta(s,a,s')} \leq \mu(s') \leq  u^\Lambda_{P_\Theta(s,a,s')}\Bigr\}.
\]
Computing $\mathcal{P}_{\Lambda(\mathcal{U})}$ requires solving $2|\Lambda|$ linear programs. 
Robust value iteration on the resulting IMDP then has the same per-iteration cost as the standard interval model,
but can yield significantly tighter uncertainty sets since the bounds reflect information aggregated across all transitions through $\mathcal{U}$. Moreover, the individual LPs are independent and can be solved in parallel when constructing the model.

\vspace{10pt}
\noindent \textbf{Parameter-wise projections.}
Expression-wise projections can be expensive when the number of distinct expressions $|\Lambda|$ is large. As an alternative, we project $\mathcal{U}$ to interval constraints on the parameters $\Theta$ themselves. Concretely, this yields an axis-aligned over-approximation of the polytope $\mathcal{U}$ by projecting onto each parameter dimension and forming the corresponding enclosing hyperrectangle.

For each parameter $\theta_i \in \Theta$, $1 \leq i \leq \ell$, we compute its  lower and upper bounds over the uncertainty region $\mathcal{U}$ by
$
\underline{\theta}_i \;\coloneqq\; \min_{\vv\in\mathcal{U}} \vv_i
$
and
$
\overline{\theta}_i \;\coloneqq\; \max_{\vv\in\mathcal{U}} \vv_i.
$
These define the smallest hyperrectangle that contains $\mathcal{U}$,
$
B(\mathcal{U})
\;\coloneqq\;
\prod_{i=1}^{\ell} [\underline{\theta}_i,\overline{\theta}_i]
\;\supseteq\;
\mathcal{U}.
$
We then derive bounds for each expression $f\in\Lambda$ by evaluating it over $B(\mathcal{U})$, 
\[
l^\Theta_f \;\coloneqq\; \min_{\vv \in B(\mathcal{U})} f[\vv],
\qquad
u^\Theta_f \;\coloneqq\; \max_{\vv \in B(\mathcal{U})} f[\vv].
\]
Since each $f$ is linear, these extrema are attained at a vertex of $B(\mathcal{U})$. Thus, computing $l^\Theta_f$ and $u^\Theta_f$ is simpler than the expression-wise projection bounds, as it does not require solving an LP over $\mathcal{U}$, but only selecting, for each parameter $\theta_i$, either $\underline{\theta}_i$ or $\overline{\theta}_i$ according to the sign of its coefficient in $f$.
The resulting interval uncertainty set $\mathcal{P}_{\Theta(\mathcal{U})}$ is then defined analogously to $\mathcal{P}_{\Lambda(\mathcal{U})}$.

Computing $\mathcal{P}_{\Theta(\mathcal{U})}$ requires solving $2|\Theta|$ linear programs over $\mathcal{U}$, which can be significantly cheaper than solving two LPs for every expression in $\Lambda$. However, the hyperrectangle $B(\mathcal{U})$ discards dependencies between parameters and may therefore yield considerably looser bounds than expression-wise projection. Moreover, the induced intervals need not be tighter than the original intervals obtained directly from the samples: the two constructions are, in general, incomparable, and either can produce tighter bounds. We give counterexamples for both directions in the proof of Theorem~\ref{thm:inclusion} in Appendix~\ref{app:proofs}. To recover inclusion of the original intervals, one can simply intersect the resulting bounds.

We now state the inclusion hierarchy of the discussed uncertainty sets. We denote by $\mathcal{P}_I$ the original interval uncertainty set obtained from the samples using parameter tying as described in Section~\ref{sec:parametertying}.

\begin{theorem}[Inclusion hierarchy]
\label{thm:inclusion}
The following inclusions hold:
\[
\mathcal{P}_I
\;\supseteq\;
\mathcal{P}_{\Lambda(\mathcal{U})}
\;\supseteq\;
\mathcal{P}_{R(\mathcal{U})}
\;\supseteq\;
\mathcal{P}_{\mathcal{U}}, \quad \text{and}
\quad
\mathcal{P}_{\Theta(\mathcal{U})}
\;\supseteq\;
\mathcal{P}_{\Lambda(\mathcal{U})}.
\]
Moreover, $\mathcal{P}_{\Theta(\mathcal{U})}$ is in general incomparable to $\mathcal{P}_I$. \qed
\end{theorem}
The inclusion hierarchy of Theorem~\ref{thm:inclusion} implies that every over-approximating relaxation inherits the same high-confidence inclusion guarantee for the true, unknown MDP as $\mathcal{P}_{\mathcal{U}}$ in Theorem~\ref{thm:soundness}. Consequently, all such relaxations can be used for robust policy synthesis with high-confidence performance guarantees.

\subsection{Linear Relaxations for Polynomial Constraints}
\label{sec:mccormick}

Thus far, we assumed that the parametric expressions in $\Lambda$ are linear, and hence that $\mathcal{U}$ is a polytope in the parameter space. In general, transition probability expressions in a pMDP may be \emph{polynomials} in the parameters, which makes the projection constraints $l_f \le f[\vv] \le u_f$ non-linear. As a consequence, the rectangular relaxations introduced in Section~\ref{sec:rectrelax} no longer reduce to linear programs. To retain tractability while preserving soundness, we construct a linear outer approximation of the feasible region using \emph{McCormick envelopes}~\cite{DBLP:journals/mp/McCormick76,DBLP:journals/jgo/ScottSB11}.

McCormick envelopes are linear inequalities that describe the convex hull of a bilinear term $z = xy$ with upper and lower bounds $\underline{x}\le x\le \overline{x}$ and $\underline{y}\le y\le \overline{y}$~\cite{DBLP:journals/mp/McCormick76}. Concretely, they are given by:
\begin{equation}
\label{eq:mccormick-bilin}
\begin{aligned}
z &\ge \underline{y}\,x + \underline{x}\,y - \underline{x}\,\underline{y},
\qquad
z \le \underline{y}\,x + \overline{x}\,y - \overline{x}\,\underline{y},
\\
z &\ge \overline{y}\,x + \overline{x}\,y - \overline{x}\,\overline{y},
\qquad
z \le \overline{y}\,x + \underline{x}\,y - \underline{x}\,\overline{y}.
\end{aligned}
\end{equation}
While McCormick envelopes provide the exact convex-hull relaxation of a single bilinear term, they can be applied compositionally to obtain a linear outer approximation of an entire polynomial constraint. For instance, for the cubic monomial $\theta_1\theta_2\theta_3$ we introduce auxiliary variables $z_{12}=\theta_1\theta_2$ and $z_{123}=z_{12}\theta_3$ for the intermediate products, and add the McCormick envelopes in~\eqref{eq:mccormick-bilin} for both bilinear terms. Applying this construction to all monomials in the constraint system defining $\mathcal{U}$ in~\eqref{eqn:constraintsys} yields a system of linear constraints that over-approximates the original polynomial relations. As an outer approximation, it preserves soundness and can be used within our rectangular relaxations.
Formal details on the construction and the resulting inequality system, are provided in Appendix~\ref{app:mccormick}.

\vspace{10pt}
\noindent \textbf{Optimisation-based bound tightening.}
The McCormick inequalities in~\eqref{eq:mccormick-bilin} require finite lower and upper bounds for every variable. 
We obtain initial bounds by solving two LPs per variable over the current linear relaxation, and then iteratively refine them: tighter bounds yield tighter McCormick envelopes, which in turn can imply further bound improvements. This iterative procedure is known as optimisation-based bound tightening (OBBT) and is a standard domain-reduction technique in global optimisation~\cite{DBLP:journals/jgo/GleixnerBMW17}. In our implementation, we apply OBBT as a preprocessing step to construct a tight linear relaxation of the uncertainty region $\mathcal{U}$. In the experimental evaluation in Section~\ref{sec:experiments}, we show that this approach yields tight relaxations and resulting robust guarantees.

\subsection{Joint Feasibility and Possible Emptiness of $\mathcal{U}$}
\label{sec:empty-R}

Unlike the interval uncertainty sets obtained in standard IMDP learning, the
induced uncertainty region $\mathcal{U}$ may be empty. This is a consequence of
enforcing \emph{global} consistency with the underlying parametric structure:
while each learned interval constraint may be individually feasible, there need
not exist a single parameter instantiation that satisfies all of them
simultaneously. In that case, the induced uncertainty set
$\mathcal{P}_{\mathcal{U}}$ is empty as well, and so are its rectangular~relaxations.

If the pMDP $M_\Theta$ correctly specifies the underlying system from which the
transition samples are obtained, 
then with probability at least $1-\delta$ the true parameter instantiation $\uu$ satisfies
all constraints, and hence lies in $\mathcal{U}$. In particular,
\[
\Pr\left[\,{\mathcal{U}\,} \neq \emptyset\right]
\;\ge\;
\Pr\left[\,\uu \in \mathcal{U}\,\right]
\;\ge\;
1-\delta, \quad \text{or equivalently, } \quad\Pr\left[\,{\mathcal{U}\,} = \emptyset\right] \le \delta.
\]
Thus, under correct specification, emptiness can occur with probability at most~$\delta$. 

\vspace{10pt}
\noindent \textbf{Interpretation.}
If ${\mathcal{U}}$ is empty, then no admissible parameter
instantiation simultaneously satisfies all learned interval constraints.
Statistically, this means that the uncertainty statements obtained from the data
are jointly incompatible with the assumed parametric structure at confidence
level $1-\delta$. Hence, emptiness can be interpreted as evidence against the
chosen pMDP model
at significance~level~$\delta$.

This observation is not just an edge case but also an advantage of our approach.
By enforcing joint consistency across all learned constraints, it can reveal
model misspecification that would be hidden by purely local~interval~uncertainty~sets.

\vspace{10pt}
\noindent \textbf{Fallback to interval uncertainty.}
If ${\mathcal{U}}$ is empty, one can safely fall back to the
original interval uncertainty set learned directly from the samples. In contrast
to ${\mathcal{U}}$, this interval uncertainty set is always
non-empty, since for each state-action pair it contains at least the empirical
estimate of the corresponding successor distribution, and thus admits at least
one feasible distribution per local choice. 

\subsection{Extension to Other Classes of Uncertainty Sets}
\label{sec:otherclasses}

We have so far assumed that the base uncertainty sets mapped through the structure of the underlying pMDP into parameter space are intervals. 
However, our approach and the inclusion hierarchy of Theorem~\ref{thm:inclusion} extend beyond intervals. In the following, we discuss two uncertainty classes that illustrate~this~extension.

\vspace{10pt}
\noindent \textbf{$L_1$ uncertainty sets.}
One important class is given by $L_1$-balls~\cite{DBLP:conf/icml/StrehlL05,DBLP:conf/qestformats/MeggendorferWW25} based on the Weissman inequality~\cite{WeissmanEtAl2003}. For each state--action pair, this yields a high-confidence uncertainty set around the empirical estimate $\tilde{P}(s,a)$ in the form of an $L_1$ ball, for which we provide details in Appendix~\ref{app:ellipsoidal}.
Since $L_1$-balls are polytopes, they fit seamlessly with our approach. 
Concretely, given empirical estimates $\tilde{P}(s,a)$ and radii $\varepsilon(s,a)$ for all $(s,a)\in S\times A$, we define the uncertainty~region
\begin{equation}
\label{eqn:constraintsys-l1}
\mathcal{U}_{L_1}
\;=\;
\bigl\{\,\vv \in \mathcal{D} \ \big|\ \|P_\Theta[\vv](s,a)-\tilde{P}(s,a)\|_1 \le \varepsilon(s,a) \,\bigr\}.
\end{equation}
Since each $P_\Theta(s,a,s')$ is linear in the parameters, the $L_1$-constraints can be expressed as linear inequalities. Hence, $\mathcal{U}_{L_1}$ is again a polytope in parameter space, and the rectangular relaxations introduced in Section~\ref{sec:rectrelax} carry over~directly.

\vspace{10pt}
\noindent \textbf{Ellipsoidal uncertainty sets.}
Another relevant class is given by \emph{ellipsoids}. These may arise as local uncertainty sets, for instance from second-order approximations of log-likelihood functions~\cite{DBLP:journals/ior/NilimG05}, or directly in parameter space~\cite{ayoub2020model,zhou2021nearly}.
In both cases, the resulting uncertainty region 
is generally no longer a polytope.
Nevertheless, the rectangular relaxations of Section~\ref{sec:rectrelax}, as well as the inclusion hierarchy,~carry~over. 

In particular, for an ellipsoidal parameter region, the local rectangular relaxation can still be obtained by projecting to state--action-wise uncertainty sets, but the corresponding inner optimisation problems are no longer linear programs. Instead, optimising a linear Bellman objective over ellipsoidal constraints yields second-order cone programs (SOCPs)~\cite{DBLP:journals/laa/LoboVBL98}. Likewise, expression-wise projections can be obtained by solving two SOCPs per expression to derive lower and upper bounds, after which robust value iteration proceeds as in the interval case. 
Details on the construction of ellipsoidal uncertainty sets are provided in Appendix~\ref{app:ellipsoidal}.

\section{Experimental Evaluation}
\vspace{-5pt}
\label{sec:experiments}

We have integrated our new methods for UMDP learning into the PRISM model checker~\cite{DBLP:conf/cav/KwiatkowskaNP11}.
We evaluate uncertainty sets from the three relaxations introduced in Section~\ref{sec:learning}:
rectangular ($\mathcal{P}_{\mathrm{R}(\mathcal{U})}$),
expression-wise ($\mathcal{P}_{\Lambda(\mathcal{U})}$),
and parameter-wise ($\mathcal{P}_{\Theta(\mathcal{U})}$)
comparing to a baseline of standard interval learning
with parameter tying ($\mathcal{P}_{\mathrm{I}}$). A comparison to ellipsoidal uncertainty sets is provided for completeness in Appendix~\ref{app:experiments}.
We compare both the tightness of the learned uncertainty sets and the computational efficiency of robust solving across a range of established benchmarks with linear and polynomial parameter structures.
Detailed descriptions of all case studies are in Appendix~\ref{app:experiments}. Table~\ref{tab:stats_same_suite} summarises~the~salient~characteristics.

\vspace{10pt}
\noindent \textbf{Experimental setup.}
We evaluate our methods in two settings. First, we compare the tightness of the uncertainty sets and the computational efficiency of constructing and solving the corresponding models from the \emph{same} data, i.e., the same set of collected transition samples $C$. To this end, we sample a total of $10^5$ trajectories from the true model under a uniform sampling policy. 


\begin{wraptable}{r}{0.52\textwidth}
\vspace{-17pt}
\centering
\caption{Salient characteristics of the benchmark instances. $\mathbb{P}$ denotes a temporal-logic satisfaction property and $\mathbb{E}$ an expected-reward property.}
\vspace{-2pt}
\setlength{\tabcolsep}{2.5pt}
\footnotesize
\renewcommand{\arraystretch}{0.9}
\resizebox{0.52\textwidth}{!}{%
\begin{tabular}{>{\centering\arraybackslash}m{2.3cm} >{\centering\arraybackslash}m{1.45cm} >{\centering\arraybackslash}m{1.1cm} >{\centering\arraybackslash}m{1.1cm} >{\centering\arraybackslash}m{0.75cm} >{\centering\arraybackslash}m{0.9cm} >{\centering\arraybackslash}m{0.75cm}}
\toprule
\textbf{Benchmark} & \textbf{Instance} & \textbf{$\mathbf{|S|}$} & \textbf{$\mathbf{|T|}$} & \textbf{$\mathbf{|\Theta|}$} & \textbf{$\mathbf{|\Lambda|}$} & \textbf{Prop.} \\
\midrule
\multirow[c]{3}{*}{Aircraft}
& (50, 10)   & 5463    & 73033   & 6 & 721    & $\mathbb{P}$ \\
& (100, 20)  & 39028   & 551598  & 6 & 1441   & $\mathbb{P}$ \\
& (200, 40)  & 294458  & 4285228 & 6 & 2881   & $\mathbb{P}$ \\
\addlinespace[3pt]

\multirow[c]{3}{*}{Betting Game}
& (50)       & 13500   & 111678  & 2 & 2613   & $\mathbb{E}$ \\
& (100)      & 52000   & 448378  & 2 & 5283   & $\mathbb{E}$ \\
& (150)      & 93825  & 817943  & 2 & 7147   & $\mathbb{E}$ \\
\addlinespace[3pt]

\multirow[c]{3}{*}{Engagement}
& (100)      & 1994    & 5962    & 5 & 61     & $\mathbb{E}$ \\
& (300)      & 5994    & 17962   & 5 & 61     & $\mathbb{E}$ \\
& (1000)     & 19994   & 59962   & 5 & 61     & $\mathbb{E}$ \\
\addlinespace[3pt]

\multirow[c]{3}{*}{Mars Rover}
& (50, 50)   & 191868  & 1003861 & 6 & 643    & $\mathbb{P}$ \\
& (75, 75)   & 550447  & 2973287 & 6 & 943    & $\mathbb{P}$ \\
& (125, 125) & 1749083 & 9591771 & 6 & 1543   & $\mathbb{P}$ \\
\addlinespace[3pt]

\multirow[c]{3}{*}{Glider}
& (21, 17)   & 357     & 5209    & 4 & 777    & $\mathbb{E}$ \\
& (51, 47)   & 2397    & 37066   & 4 & 4977   & $\mathbb{E}$ \\
& (106, 99)  & 4757    & 74299   & 4 & 9768   & $\mathbb{E}$ \\
\addlinespace[3pt]

\multirow[c]{3}{*}{Parallel Betting}
& (5)        & 3051    & 15543   & 3 & 477    & $\mathbb{E}$ \\
& (10)       & 21006   & 187730  & 3 & 1077   & $\mathbb{E}$ \\
& (20)       & 151556  & 1846840 & 3 & 185209 & $\mathbb{E}$ \\
\bottomrule
\end{tabular}
}
\label{tab:stats_same_suite}
\vspace{-15pt}
\end{wraptable}

In the second setting, we evaluate the different approaches in an online learning scenario. Here, trajectories are sampled sequentially, the uncertain model is periodically updated from the data collected so far, and the resulting policy is used to guide further exploration. Concretely, we follow the \emph{optimism in the face of uncertainty} principle~\cite{DBLP:conf/nips/SuilenS0022,DBLP:conf/nips/AuerJO08,DBLP:conf/nips/FruitPLB17}: whenever the count $\#(s,a)$ doubles for some state--action pair $(s,a)$, we recompute the uncertain model from all trajectories collected so far and update the sampling policy to the \emph{optimistic} policy, i.e., the policy that is optimal for the current uncertain model under the assumption that the environment resolves uncertainty in the agent's favour. 
This is a standard mechanism for balancing exploration and exploitation in policy learning. We emphasise, however, that our approach is agnostic to the concrete sampling procedure.
Once an overall sampling budget is exhausted, we use all collected trajectories to construct the final uncertain model of the respective class and synthesise the corresponding \emph{robust} policy.
So, unlike the first setting, the different approaches are not evaluated on a fixed common data set. Instead, the data collection itself is guided by the evolving uncertainty structure, so tighter uncertainty sets can lead to more effective exploration and, ultimately, to improved final policies and guarantees. 
All uncertain models are built at PAC confidence level $1-\delta = 0.999$.

\begin{table}[t]
\centering
\caption{Comparison of the final bounds and runtimes obtained after sampling $10^5$ trajectories uniformly. For each method, $[\underline{V},\overline{V}]$ denotes the final lower and upper bound on the value of the optimal policy in the true model, and $\mathrm{Rel.\ gap}$ the interval width relative to the true value $V^\ast$. $\mathrm{TO}$ indicates timeout at 2 [h].}
\vspace{5pt}
\setlength{\tabcolsep}{0.5pt}
\small
\resizebox{\textwidth}{!}{%
\begin{tabular}{>{\centering\arraybackslash}m{2.35cm} >{\centering\arraybackslash}m{1.65cm} >{\centering\arraybackslash}m{2.0cm} >{\centering\arraybackslash}m{0.95cm} >{\centering\arraybackslash}m{0.95cm} @{\hspace{9pt}} >{\centering\arraybackslash}m{2.0cm} >{\centering\arraybackslash}m{0.95cm} >{\centering\arraybackslash}m{0.95cm} @{\hspace{9pt}} >{\centering\arraybackslash}m{2.0cm} >{\centering\arraybackslash}m{0.95cm} >{\centering\arraybackslash}m{0.95cm} @{\hspace{9pt}} >{\centering\arraybackslash}m{2.0cm} >{\centering\arraybackslash}m{0.95cm} >{\centering\arraybackslash}m{0.95cm}}
\toprule
\multirow[c]{2}{*}{\textbf{Benchmark}} & \multirow[c]{2}{*}{\textbf{Instance}} & \multicolumn{3}{c}{$\pmb{\mathcal{P}_I}$} & \multicolumn{3}{c}{$\pmb{\mathcal{P}_{\Theta(\mathcal{U})}}$} & \multicolumn{3}{c}{$\pmb{\mathcal{P}_{\Lambda(\mathcal{U})}}$} & \multicolumn{3}{c}{$\pmb{\mathcal{P}_{R(\mathcal{U})}}$} \\
\cmidrule(lr){3-5} \cmidrule(lr){6-8} \cmidrule(lr){9-11} \cmidrule(lr){12-14}
 &  & $\mathbf{[\underline{V},\overline{V}]}$ & \textbf{Rel. gap} & \textbf{Time [s]} & $\mathbf{[\underline{V},\overline{V}]}$ & \textbf{Rel. gap} & \textbf{Time [s]} & $\mathbf{[\underline{V},\overline{V}]}$ & \textbf{Rel. gap} & \textbf{Time [s]} & $\mathbf{[\underline{V},\overline{V}]}$ & \textbf{Rel. gap} & \textbf{Time [s]} \\
\midrule
\multirow[c]{3}{*}{\shortstack[c]{Aircraft}} 
 & (50, 10) & $[0.67,\,0.814]$ & 0.19 & 4.04 & $[0.618,\,0.828]$ & 0.28 & 4.62 & $[0.719,\,0.767]$ & 0.06 & 5.72 & $[0.725,\,0.761]$ & \textbf{0.05} & 47.12 \\
 & (100, 20) & $[0.745,\,0.906]$ & 0.19 & 30.45 & $[0.679,\,0.918]$ & 0.29 & 34.71 & $[0.805,\,0.865]$ & 0.07 & 39.54 & $[0.814,\,0.858]$ & \textbf{0.05} & 353.36 \\
 & (200, 40) & $[0.628,\,0.92]$ & 0.36 & 688.27 & $[0.55,\,0.928]$ & 0.47 & 878.04 & $[0.754,\,0.85]$ & 0.12 & 882.78 & $[0.769,\,0.838]$ & \textbf{0.09} & 4464.25 \\
\addlinespace[4pt]
\multirow[c]{3}{*}{\shortstack[c]{Betting Game}} & (50) & $[14.4,\,65.3]$ & 1.88 & 2.70 & $[25.8,\,28.7]$ & \textbf{0.10} & 6.49 & $[25.8,\,28.7]$ & \textbf{0.10} & 7.53 & $[25.8,\,28.7]$ & \textbf{0.10} & 2.29 \\
 & (100) & $[16.5,\,162]$ & 3.24 & 13.07 & $[41.8,\,47.7]$ & \textbf{0.13} & 12.86 & $[41.8,\,47.7]$ & \textbf{0.13} & 26.31 & $[41.8,\,47.7]$ & \textbf{0.13} & 17.05 \\
 & (150) & $[17.5,\,263]$ & 3.98 & 25.18 & $[56.5,\,65.5]$ & \textbf{0.15} & 37.98 & $[56.5,\,65.5]$ & \textbf{0.15} & 53.16 & $[56.5,\,65.5]$ & \textbf{0.15} & 38.34 \\
\addlinespace[4pt]
\multirow[c]{3}{*}{Engagement} & (100) & $[38.4,\,49.3]$ & 0.26 & 5.46 & $[16.6,\,5422]$ & 127.35 & 4.44 & $[39.8,\,45]$ & \textbf{0.12}& 1.00 & $[39.8,\,45]$ & \textbf{0.12} & 1.82 \\
 & (300) & $[38.6,\,46.8]$ & 0.19 & 16.55 & $[18.7,\,3615]$ & 84.74 & 21.30 & $[40.8,\,45.3]$ & \textbf{0.11} & 7.95 & $[40.8,\,45.3]$ & \textbf{0.11} & 7.27 \\
 & (1000) & $[39.1,\,45.4]$ & 0.15 & 70.02 & $[24,\,687]$ & 15.63 & 57.74 & $[41.3,\,43.8]$ & \textbf{0.06} & 65.80 & $[41.3,\,43.8]$ & \textbf{0.06} & 95.44 \\
\addlinespace[4pt]
\multirow[c]{3}{*}{\shortstack[c]{Mars Rover}}
 & (50, 50) & $[0,\,0.797]$ & 1.59 & 92.02 & $[0.177,\,0.833]$ & 1.31 & 136.87 & $[0.488,\,0.515]$ & \textbf{0.05} & 115.15 & -- & -- & TO \\
 & (75, 75) & $[0,\,0.895]$ & 1.42 & 521.67 & $[0.0904,\,0.9]$ & 1.28 & 576.92 & $[0.619,\,0.647]$ & \textbf{0.04} & 580.21 & -- & -- & TO \\
  & (125, 125) & $[0,\,0.877]$ & 1.42 & 2169.21 & $[0.0441,\,0.922]$ & 1.42 & 2553.84 & $[0.599,\,0.635]$ & \textbf{0.06} & 2306.29 & -- & -- & TO \\
  
  \addlinespace[4pt]
  \multirow[c]{3}{*}{Glider} & (21, 17) & $[41.1,\,44.6]$ & 0.08 & 0.25 & $[42.4,\,42.9]$ & 0.01 & 0.19 & $[42.5,\,42.8]$ & \textbf{0.01} & 0.71 & $[42.5,\,42.8]$ & \textbf{0.01} & 37.26 \\
 & (51, 47) & $[62.3,\,20001]$ & 293.07 & 28.92 & $[58.8,\,59.2]$ & 0.01 & 1.67 & $[58.8,\,59.1]$ & \textbf{0.01} & 17.16 & $[58.8,\,59.1]$ & \textbf{0.01} & 6268.23 \\
 & (106, 99) & $[66.4,\,35254]$ & 464.04 & 52.95 & $[75.6,\,76.1]$ & 0.01 & 3.88 & $[75.6,\,76]$ & \textbf{0.01} & 48.50 & -- & -- & TO \\
  \addlinespace[4pt]
 \multirow[c]{3}{*}{\shortstack[c]{Parallel Betting}} & (5) & $[24,\,29.4]$ & 0.20 & 0.31 & $[25.1,\,28.2]$ & 0.12 & 0.40 & $[26.1,\,27.2]$ & \textbf{0.04} & 0.80 & $[26.1,\,27.2]$ & \textbf{0.04} & 10.16 \\
 & (10) & $[26.4,\,40.6]$ & 0.44 & 2.86 & $[29.7,\,34.5]$ & 0.15 & 3.23 & $[31.1,\,32.9]$ & 0.06 & 6.42 & $[31.1,\,32.8]$ & \textbf{0.05} & 116.52 \\
 & (20) & $[13,\,110]$ & 3.39 & 21.23 & $[25.4,\,31.7]$ & \textbf{0.22} & 666.30 & -- & -- & TO & -- & -- & TO \\

\bottomrule
\end{tabular}
} 
\label{tab:results_udtmc_optpolicy_interval_relgap_total}
\vspace{-14pt}
\end{table}

\vspace{10pt}
\noindent \textbf{Results.}
Table~\ref{tab:results_udtmc_optpolicy_interval_relgap_total} reports the results for the first setting 
under uniform sampling. 
For each uncertainty structure, we evaluate how well the resulting UMDP can certify the performance of the policy that is optimal in the true MDP. We report the robust value $\underline{V}$ and the optimistic value $\overline{V}$ of this policy, corresponding to the worst- and best-case value that can be certified under the respective uncertainty structure. This allows a fair comparison of the different approximations. In addition, we report the \emph{relative gap}, i.e., the width of the interval $[\underline{V},\overline{V}]$ divided by the true value $V^*$, as well as the total runtime required to construct and solve the respective uncertain models. Appendix~\ref{app:experiments} provides extended results, including a breakdown of the runtime spent on model construction and solving.

Figure~\ref{fig:main} illustrates the online policy learning process for two representative benchmarks. For each uncertainty structure, we plot the number of processed trajectories against (1) the performance of the robust policy synthesised after that number of trajectories, evaluated in the true (hidden) MDP, shown as solid lines, and (2) the corresponding PAC performance guarantee, shown as dashed lines. Appendix~\ref{app:experiments} provides the extended results for all considered~case~studies.

\begin{figure}[t]
\setlength{\textfloatsep}{8pt plus 2pt minus 2pt}
\setlength{\floatsep}{6pt plus 2pt minus 2pt}
\setlength{\intextsep}{8pt plus 2pt minus 2pt}
    \centering

    \begin{subfigure}[b]{\textwidth}
        \centering
        \includegraphics[width=0.95\linewidth]{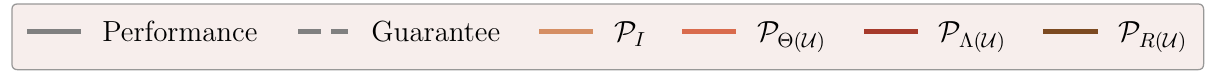}
        \label{fig:sub3}
    \end{subfigure}

    \vspace{0.4em}

    \begin{subfigure}[b]{0.46\textwidth}
        \centering
        \includegraphics[width=\linewidth]{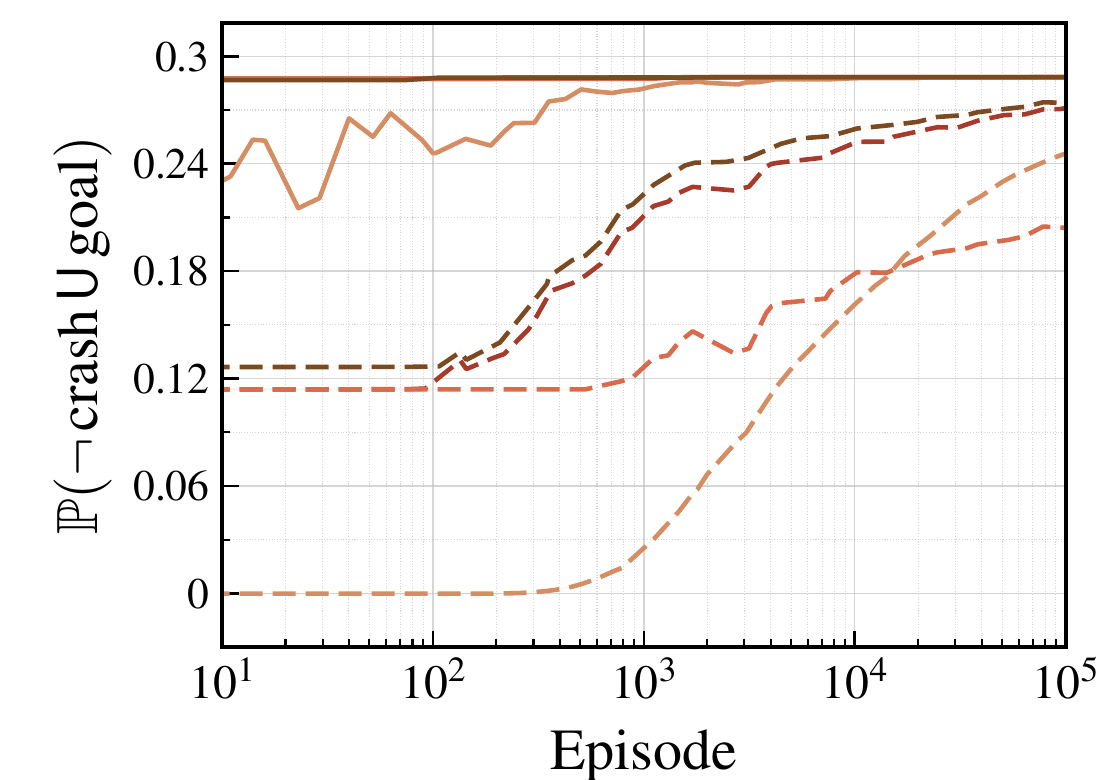}
        \caption{Aircraft}
        \label{fig:sub1}
    \end{subfigure}
    \hspace{4pt}
    \begin{subfigure}[b]{0.46\textwidth}
        \centering
        \includegraphics[width=\linewidth]{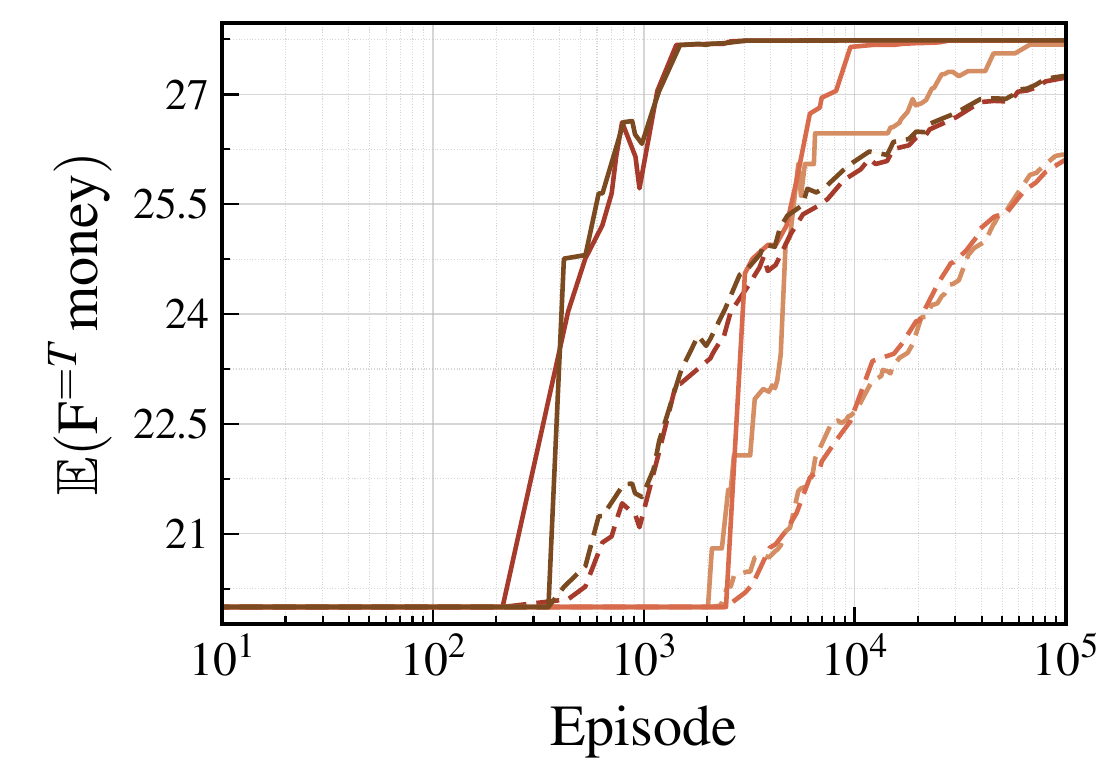}
        \caption{Parallel Betting}
        \label{fig:sub2}
    \end{subfigure}

\caption{Robust policy learning progress in the online setting. Solid lines show the performance of the current robust policy on the true model. Dashed lines show the corresponding PAC performance guarantee certified for the policy.}
    \label{fig:main}
    \vspace{-15pt}
\end{figure}

\vspace{10pt}
\noindent \textbf{Discussion.}
In the first setting, where all uncertainty sets are constructed from the same data, the results confirm the behaviour predicted by the inclusion hierarchy of Theorem~\ref{thm:inclusion}. As expected, the rectangular relaxation $\mathcal{P}_{R(\mathcal{U})}$ of the uncertainty region $\mathcal{U}$ yields the tightest approximation and, consequently, the sharpest certification of policy performance, but at the highest computational cost.
In low-dimensional domains such as the Betting Game, $\mathcal{P}_{R(\mathcal{U})}$ remains practically efficient via vertex enumeration, but this becomes ineffective~in~higher~dimensions.

Further, the results show that the expression-wise projection $\mathcal{P}_{\Lambda(\mathcal{U})}$ remains tight while retaining computational efficiency. In many cases, it achieves exactly the same values as $\mathcal{P}_{R(\mathcal{U})}$, but at a cost comparable to the substantially less tight
baseline interval uncertainty set $\mathcal{P}_{I}$ obtained from parameter tying alone. The main limitation arises when the number of distinct expressions $|\Lambda|$ is very large, since constructing $\mathcal{P}_{\Lambda(\mathcal{U})}$ then requires solving a correspondingly large number of LPs.
The experiments also illustrate the incomparability of $\mathcal{P}_{I}$ and $\mathcal{P}_{\Theta(\mathcal{U})}$: the latter can be tighter in some instances at low computational cost, but is overall dominated by the more consistently effective expression-wise projection.
Overall, $\mathcal{P}_{\Lambda(\mathcal{U})}$ appears to offer the best balance between computational efficiency and approximation tightness. By leveraging the parametric structure, it achieves substantially tighter approximations than the baseline $\mathcal{P}_{I}$ based solely on parameter tying, in several cases tightening the bounds by orders of magnitude.

We observe a similar pattern in the second setting, where robust policies are learned online from iteratively collected data and progressively refined uncertainty models. Learning with $\mathcal{P}_{R(\mathcal{U})}$ consistently yields the tightest and strongest performance guarantees throughout the learning process. In particular, it improves sample efficiency, measured by the certified performance guarantee achieved after a given number of processed trajectories, by up to an order of magnitude compared to the baseline approach based on $\mathcal{P}_{I}$. The expression-wise projection $\mathcal{P}_{\Lambda(\mathcal{U})}$ again closely tracks the guarantees obtained via $\mathcal{P}_{R(\mathcal{U})}$, while the results also empirically illustrate the incomparability of $\mathcal{P}_{\Theta(\mathcal{U})}$ and $\mathcal{P}_{I}$.

The evaluation demonstrates that exploiting the parametric structure can lead to substantial improvements in both the tightness of the learned uncertainty sets and the quality of the resulting robust policies and their certified performance guarantees, in several cases by orders of magnitude. 
Overall, this highlights the central benefit of our approach: by lifting statistical uncertainty from the transitions to the parameter space and reasoning about it jointly through the pMDP structure, we obtain substantially less conservative robust learning and synthesis than interval-based uncertain MDP learning alone, while maintaining formal PAC guarantees on performance and robustness.

\section{Conclusions}
\vspace{-5pt}
\label{sec:conclusion}

We have presented a framework for robust parameter learning in pMDPs that lifts statistical uncertainty from individual transitions into parameter space. By exploiting the algebraic dependencies induced by shared parameters, the approach yields tighter uncertainty sets than traditional UMDP learning while retaining PAC guarantees, and admits tractable relaxations for robust policy synthesis. Our experiments show that this leads to substantially stronger certified guarantees and more effective robust policies.

A promising direction for future work is to combine robust parameter learning with abstractions. Since the parametric model is available, symbolic probabilistic bisimulation techniques~\cite{DBLP:conf/popl/LarsenS89,DBLP:conf/csfw/Smith03} could be used to obtain smaller abstract models, enabling more efficient verification and policy synthesis. Data collected in the original system could then be mapped to the abstraction, where our learning procedure can be applied. Combining this with recent techniques for learning abstractions~\cite{DBLP:conf/cav/AbateGS24,DBLP:conf/cav/AbateGMS25,DBLP:conf/birthday/AbateGRS25} may enable robust policy synthesis with formal PAC performance guarantees for very large, and potentially infinite-state systems.

\bibliographystyle{plain}
\bibliography{references}

\appendix

\section{Proofs}
\label{app:proofs}
We provide the formal proofs for the Theorems~\ref{thm:soundness} and~\ref{thm:inclusion}.

\subsection{Proof of Theorem~\ref{thm:soundness}}

\begin{theoremrep}[Soundness of $\mathcal{U}$, restated]
With probability at least $1-\delta$, the true parameter instantiation $\uu$ satisfies $\uu \in \mathcal{U}$. Equivalently,
\[
\Pr\bigl[\;P[\uu] \in \mathcal{P}_{\mathcal{U}}\;\bigr] \;\ge\; 1-\delta.
\]
Consequently, for any policy $\pi \in \Pi$ and state $s \in S$, with probability at least $1-\delta$ the robust value under the induced UMDP lower-bounds the true value in $M$, i.e.,
\[
\Pr\bigl[\;V^\pi_{U}(s) \leq V^{\pi}_{M}(s)\;\bigr] \;\ge\; 1-\delta.
\]
\end{theoremrep}

\begin{proof}
For each non-constant expression $f \in \Lambda_{\mathsf{unk}}$, the learned interval satisfies
\[
\Pr\bigl[\, l_f \leq f[\uu] \leq u_f \,\bigr]
\;\ge\;
1-\frac{\delta}{|\Lambda_{\mathsf{unk}}|}.
\]
For every constant expression $f \in \Lambda \setminus \Lambda_{\mathsf{unk}}$, we set $l_f=u_f=f$, and hence $l_f \leq f[\uu] \leq u_f$ holds deterministically.y
Now consider the event
\[
E \;:=\; \bigl\{\, \forall f \in \Lambda:\ l_f \leq f[\uu] \leq u_f \,\bigr\}.
\]
If $E$ does not occur, then at least one of the intervals for some $f \in \Lambda_{\mathsf{unk}}$ fails to contain the true value $f[\uu]$. Thus
\[
E^c
\;\subseteq\;
\bigcup_{f \in \Lambda_{\mathsf{unk}}}
\bigl\{\, f[\uu] \notin [l_f,u_f] \,\bigr\},
\]
and by the union bound,
\[
\Pr(E^c)
\;\le\;
\sum_{f \in \Lambda_{\mathsf{unk}}}
\Pr\bigl[\, f[\uu] \notin [l_f,u_f] \,\bigr]
\;\le\;
\sum_{f \in \Lambda_{\mathsf{unk}}} \frac{\delta}{|\Lambda_{\mathsf{unk}}|}
\;=\;
\delta.
\]
Hence $\Pr(E) \ge 1-\delta$. By definition of $\mathcal{U}$, the event $E$ is exactly the event $\uu \in \mathcal{U}$.

For the second claim, note that by definition of $\mathcal{P}_{\mathcal{U}}$, we have $P[\uu] \in \mathcal{P}_{\mathcal{U}}$ if and only if $\uu \in \mathcal{U}$. Therefore,
\[
\Pr\bigl[\,P[\uu] \in \mathcal{P}_{\mathcal{U}}\,\bigr]
=
\Pr\bigl[\,\uu \in \mathcal{U}\,\bigr]
\;\ge\;
1-\delta.
\]

For the third claim, fix any policy $\pi \in \Pi$ and state $s \in S$. On the event $P[\uu] \in \mathcal{P}_{\mathcal{U}}$, the true kernel is admissible, and therefore
\[
V_U^\pi(s)
=
\inf_{P' \in \mathcal{P}_{\mathcal{U}}} V_{P'}^\pi(s)
\;\le\;
V_{P[\uu]}^\pi(s)
=
V_M^\pi(s).
\]
Hence the event $\{P[\uu] \in \mathcal{P}_{\mathcal{U}}\}$ implies the event $\{V_U^\pi(s) \le V_M^\pi(s)\}$, and so
\begin{align*}
\Pr\bigl[\,V_U^\pi(s) \leq V_M^\pi(s)\,\bigr]
&\;\ge\;
\Pr\bigl[\,P[\uu] \in \mathcal{P}_{\mathcal{U}}\,\bigr]
\;\ge\;
1-\delta.
\end{align*}
\end{proof}

\subsection{Proof of Theorem~\ref{thm:inclusion}}
\begin{theoremrep}[Inclusion hierarchy, restated]
\label{thm:inclusion-restated}
The following inclusions hold:
\[
\mathcal{P}_I
\;\supseteq\;
\mathcal{P}_{\Lambda(\mathcal{U})}
\;\supseteq\;
\mathcal{P}_{R(\mathcal{U})}
\;\supseteq\;
\mathcal{P}_{\mathcal{U}},
\qquad\text{and}\qquad
\mathcal{P}_{\Theta(\mathcal{U})}
\;\supseteq\;
\mathcal{P}_{\Lambda(\mathcal{U})}.
\]
Moreover, $\mathcal{P}_{\Theta(\mathcal{U})}$ is in general incomparable to $\mathcal{P}_I$.
\end{theoremrep}

\begin{proof}
We prove each inclusion individually.

\paragraph{1. $\mathcal{P}_{R(\mathcal{U})} \supseteq \mathcal{P}_{\mathcal{U}}$.}
Let $P[\vv] \in \mathcal{P}_{\mathcal{U}}$ for some $\vv \in \mathcal{U}$. Then, by definition of the local projections, we have $P[\vv](s,a) \in \mathcal{P}_{R(\mathcal{U})}(s,a)$ for every $(s,a)\in S\times A$. Hence $P[\vv] \in \mathcal{P}_{R(\mathcal{U})}$, which shows $\mathcal{P}_{\mathcal{U}} \subseteq \mathcal{P}_{R(\mathcal{U})}$.

\paragraph{2. $\mathcal{P}_{\Lambda(\mathcal{U})} \supseteq \mathcal{P}_{R(\mathcal{U})}$.}
Fix a state--action pair $(s,a)$ and let $\mu \in \mathcal{P}_{R(\mathcal{U})}(s,a)$. By definition, there exists $\vv \in \mathcal{U}$ such that $\mu = P[\vv](s,a)$. For each successor $s' \in S$, let $f = P_\Theta(s,a,s') \in \Lambda$. Since $l_f^\Lambda$ and $u_f^\Lambda$ are defined as the minimum and maximum of $f[\vv']$ over all $\vv' \in \mathcal{U}$, we have $l_f^\Lambda \le f[\vv] = \mu(s') \le u_f^\Lambda$. Thus $\mu \in \mathcal{P}_{\Lambda(\mathcal{U})}(s,a)$. Since $(s,a)$ was arbitrary, it follows that $\mathcal{P}_{R(\mathcal{U})} \subseteq \mathcal{P}_{\Lambda(\mathcal{U})}$.

\paragraph{3. $\mathcal{P}_{I} \supseteq \mathcal{P}_{\Lambda(\mathcal{U})}$.}
Recall that $\mathcal{U}$ is defined by the original learned bounds, i.e.,
\[
\mathcal{U}
=
\bigl\{\,\vv \in \mathcal{D} \mid \forall f\in\Lambda:\ l_f \le f[\vv] \le u_f\,\bigr\}.
\]
Hence, for every $f\in\Lambda$, all values of $f[\vv]$ with $\vv\in\mathcal{U}$ lie in the interval $[l_f,u_f]$. Since $l_f^\Lambda = \min_{\vv\in\mathcal{U}} f[\vv]$ and $u_f^\Lambda = \max_{\vv\in\mathcal{U}} f[\vv]$, it follows that
\[
l_f \le l_f^\Lambda \le u_f^\Lambda \le u_f.
\]
Therefore, for every state--action pair $(s,a)$ and successor $s'\in S$, the interval induced by expression-wise projection is contained in the original interval. Thus $\mathcal{P}_{\Lambda(\mathcal{U})}(s,a) \subseteq \mathcal{P}_{I}(s,a)$ for all $(s,a)$, and consequently $\mathcal{P}_{\Lambda(\mathcal{U})} \subseteq \mathcal{P}_{I}$.

\paragraph{4. $\mathcal{P}_{\Theta(\mathcal{U})} \supseteq \mathcal{P}_{\Lambda(\mathcal{U})}$.}
Since $B(\mathcal{U})$ is the smallest axis-aligned hyperrectangle containing $\mathcal{U}$, we have $\mathcal{U} \subseteq B(\mathcal{U})$. Therefore, for every expression $f\in\Lambda$, minimizing or maximizing over the larger set $B(\mathcal{U})$ can only decrease the lower bound and increase the upper bound. In other words, $l_f^\Theta \le l_f^\Lambda$ and $u_f^\Theta \ge u_f^\Lambda$. Thus the interval induced by parameter-wise projection contains the corresponding interval induced by expression-wise projection, and so $\mathcal{P}_{\Lambda(\mathcal{U})}(s,a) \subseteq \mathcal{P}_{\Theta(\mathcal{U})}(s,a)$ for all $(s,a)$. Hence $\mathcal{P}_{\Lambda(\mathcal{U})} \subseteq \mathcal{P}_{\Theta(\mathcal{U})}$.

\paragraph{5. Incomparability of $\mathcal{P}_{\Theta(\mathcal{U})}$ and $\mathcal{P}_{I}$.}
We give two counterexamples showing that neither inclusion holds in general.

\smallskip
\noindent\emph{$\mathcal{P}_{\Theta(\mathcal{U})}\not\subseteq \mathcal{P}_I$.}
Consider two parameters $\theta_1,\theta_2\in[0,1]$ and three transition expressions
\[
f_1(\theta)=\theta_1,\qquad
f_2(\theta)=\theta_2,\qquad
f_3(\theta)=\tfrac12(\theta_1+\theta_2),
\]
arising from distinct state--action pairs. Suppose the learned intervals are $f_1 \in [0,1]$, $f_2 \in [0,1]$, and $f_3 \in [\tfrac12,\tfrac12]$. Then the corresponding uncertainty region is
\[
\mathcal{U}
=
\bigl\{\,(\theta_1,\theta_2)\in[0,1]^2 \mid \theta_1+\theta_2=1\,\bigr\},
\]
and thus $B(\mathcal{U})=[0,1]^2$. For the expression $f_3$, expression-wise projection yields the exact interval $[l^\Lambda_{f_3},u^\Lambda_{f_3}] = [\tfrac12,\tfrac12]$, whereas parameter-wise projection over $B(\mathcal{U})$ yields $[l^\Theta_{f_3},u^\Theta_{f_3}] = [0,1]$. Hence $\mathcal{P}_{\Theta(\mathcal{U})}$ admits kernels whose $f_3$-component lies outside the original interval, showing that $\mathcal{P}_{\Theta(\mathcal{U})}\not\subseteq \mathcal{P}_I$.

\smallskip
\noindent\emph{$\mathcal{P}_{I}\not\subseteq \mathcal{P}_{\Theta(\mathcal{U})}$.}
Consider now a pMDP with one parameter $\theta\in[0,1]$ and two distinct state--action pairs with transition expressions $f_1(\theta)=\theta$ and $f_2(\theta)=1-\theta$. Suppose the learned intervals are $f_1 \in [0.4,0.7]$ and $f_2 \in [0.4,0.7]$. Then
\[
\mathcal{U}
=
\bigl\{\,\theta\in[0,1] \mid 0.4\le\theta\le0.7,\; 0.4\le 1-\theta\le0.7\,\bigr\}
=
[0.4,0.6].
\]
Thus $B(\mathcal{U})=[0.4,0.6]$, and parameter-wise projection yields the refined intervals $f_1,f_2 \in [0.4,0.6]$. In contrast, the original interval uncertainty set $\mathcal{P}_I$ still permits the independent choices $f_1=0.7$ and $f_2=0.7$ at the two distinct state--action pairs. Such a kernel belongs to $\mathcal{P}_I$ but not to $\mathcal{P}_{\Theta(\mathcal{U})}$. Hence $\mathcal{P}_{I}\not\subseteq \mathcal{P}_{\Theta(\mathcal{U})}$.
\end{proof}

\section{Correspondence between pMDPs and UMDPs}
\label{app:pmdpumdp}

We formalise the relationship between parametric MDPs and uncertain MDPs.

\subsection{From pMDPs to UMDPs}

Let $M_\Theta = (S, A, s_0, \Theta, P_\Theta)$ be a pMDP and let $\mathcal{U} \subseteq \mathbb{R}^{|\Theta|}$ be a set of admissible parameter instantiations.
Each $\mathbf{u} \in \mathcal{U}$ induces a concrete transition kernel $P[\mathbf{u}] \in (\Delta(S))^{S \times A}$.
We define the UMDP induced by $M_\Theta$ and $\mathcal{U}$ as
\[
U_{\mathcal{U}} \coloneqq (S, A, s_0, \mathcal{P}_{\mathcal{U}}),
\]
where the uncertainty set is given by
\[
\mathcal{P}_{\mathcal{U}}
\;\coloneqq\;
\bigl\{\, P[\mathbf{u}] \in (\Delta(S))^{S \times A} \mid \mathbf{u} \in \mathcal{U} \,\bigr\}.
\]

The structure of the uncertainty set $\mathcal{P}_{\mathcal{U}}$ is determined by the parameter dependencies in the transition function $P_\Theta$ and by the shape of the admissible instantiation set $\mathcal{U}$.
In particular, shared parameters induce coupling between transition probabilities at different state--action pairs.

\vspace{10pt}
\noindent \textbf{Coupled uncertainty.}
If there exists a parameter $\theta \in \Theta$ that occurs in the transition polynomials of two distinct state--action pairs $(s,a)$ and $(s',a')$, then instantiating $v$ simultaneously constrains both $P(s,a)$ and $P(s',a')$.
In this case, the resulting uncertainty set $\mathcal{P}_{\mathcal{U}}$ is, in general, non-rectangular.

\vspace{10pt}
\noindent \textbf{Rectangular uncertainty.}
Suppose that the parameter set $\Theta$ can be partitioned as $\Theta = \biguplus_{(s,a)\in S\times A} \Theta_{s,a}$ such that $P(s,a)$ depends only on parameters in $\Theta_{s,a}$, and that the admissible instantiation set factorises as
\[
\mathcal{U} = \prod_{(s,a)\in S\times A} \mathcal{U}_{s,a},
\qquad
\mathcal{U}_{s,a} \subseteq \mathbb{R}^{|\Theta_{s,a}|}.
\]
Then the induced uncertainty set is $(s,a)$-rectangular and can be written as
\[
\mathcal{P}_{\mathcal{U}}
=
\prod_{(s,a)\in S\times A}
\bigl\{\, P(s,a)[\mathbf{u}_{s,a}] \mid \mathbf{u}_{s,a} \in \mathcal{U}_{s,a} \,\bigr\},
\]
where $\mathbf{u}_{s,a}$ denotes the restriction of $\mathbf{u}$ to $\Theta_{s,a}$.

\subsection{From UMDPs to pMDPs}

Conversely, any UMDP can be represented as a pMDP together with a suitable admissible parameter set.
Let $U = (S, A, s_0, \mathcal{P})$ be a UMDP, where $\mathcal{P} \subseteq (\Delta(S))^{S \times A}$ is a non-empty set of admissible transition kernels.

\vspace{10pt}
\noindent \textbf{Rectangular case.}
Assume that the uncertainty set $\mathcal{P}$ is $(s,a)$-rectangular, i.e.,
\[
\mathcal{P}
=
\prod_{(s,a)\in S\times A} \mathcal{P}(s,a)
\qquad\text{with}\qquad
\mathcal{P}(s,a) \subseteq \Delta(S).
\]

We construct an equivalent pMDP by introducing explicit parameters for the transition probabilities.
For each $(s,a)\in S\times A$ and each $s'\in S$, let $\theta_{s,a,s'}$ be a parameter representing the probability of transitioning from $(s,a)$ to $s'$.
Let
\[
\Theta \coloneqq \{\, \theta_{s,a,s'} \mid (s,a)\in S\times A,\ s'\in S \,\}.
\]
Define the parametric transition function $P_\Theta$ by
\[
P_\Theta(s,a)(s') \coloneqq \theta_{s,a,s'}.
\]
The admissible parameter instantiations are restricted to
\[
\mathcal{U}
\;\coloneqq\;
\Bigl\{\, \mathbf{u} \in \mathbb{R}^{|\Theta|}
\ \Big|\
\bigl( u(\theta_{s,a,s'}) \bigr)_{s'\in S} \in \mathcal{P}(s,a)
\ \text{for all } (s,a)\in S\times A
\Bigr\}.
\]

By construction, each $\mathbf{u}\in\mathcal{U}$ induces a transition kernel $P[\mathbf{u}] \in \mathcal{P}$, and every kernel in $\mathcal{P}$ arises from a unique instantiation $\mathbf{u}\in\mathcal{U}$.
Hence, the pMDP together with $\mathcal{U}$ induces exactly the UMDP $U$.

\vspace{10pt}
\noindent \textbf{Coupled case.}
Assume that the uncertainty set $\mathcal{P} \subseteq (\Delta(S))^{S\times A}$ is non-rectangular.
We again construct a pMDP by introducing explicit parameters for the transition probabilities, but without imposing any independence assumptions.

For each $(s,a)\in S\times A$ and $s'\in S$, let $\theta_{s,a,s'}$ be a parameter representing the probability of transitioning from $(s,a)$ to $s'$.
Let
\[
\Theta \coloneqq \{\, \theta_{s,a,s'} \mid (s,a)\in S\times A,\ s'\in S \,\}.
\]
Define the parametric transition function $P$ by
\[
P_\Theta(s,a)(s') \coloneqq \theta_{s,a,s'}.
\]
The admissible parameter instantiations are restricted to
\[
\mathcal{U}
\;\coloneqq\;
\Bigl\{\, \mathbf{u} \in \mathbb{R}^{|\Theta|}
\ \Big|\
P[\mathbf{u}] \in \mathcal{P}
\Bigr\}.
\]

By construction, $\mathcal{U}$ encodes all global constraints and dependencies between transition probabilities captured by $\mathcal{P}$.
Moreover, the pMDP together with $\mathcal{U}$ induces exactly the uncertainty set $\mathcal{P}$, and is non-rectangular whenever $\mathcal{P}$ is.

Together, the constructions show that pMDPs and UMDPs are equivalent modelling formalisms when paired with a suitable admissible parameter set.

\section{Parameter Tying}
\label{app:parametertying}
This appendix gives the formal details of the parameter-tying preprocessing step used in Section~\ref{sec:parametertying}. The key idea is that transitions labelled by the same parametric expression correspond to the same unknown probability under any parameter instantiation and can therefore be learned from pooled counts.
We assume throughout that no state--action distribution assigns the same parametric expression to two distinct successors, i.e.,
\[
P_\Theta(s,a,s') = P_\Theta(s,a,s'')
\quad\Longrightarrow\quad
s'=s''.
\]
This assumption is made only for notational and technical convenience. If the same expression labels two distinct successors of a fixed state--action pair, then these occurrences do not give rise to separate Bernoulli events, but are coupled as part of the same successor distribution. Such cases can be eliminated without loss of generality: if two distinct successors $s',s''$ of the same state--action pair carry the same expression, we introduce an intermediate state that is reached with the shared probability, followed by a fixed distribution over the original successors. This preserves the relevant parametric structure and probabilities while ensuring that each expression labels at most one successor per state--action~pair~\cite{DBLP:journals/corr/abs-2507-15532}.

For an expression $f\in\Lambda$, we define its occurrence set by
\[
\mathrm{Occ}(f)
\;:=\;
\bigl\{\, (s,a,s') \in S\times A\times S \mid P_\Theta(s,a,s') = f \,\bigr\}.
\]
All transitions in $\mathrm{Occ}(f)$ share the same parametric expression and hence the true probability value $f[\uu]$. Consequently, they correspond to repeated observations of the same Bernoulli event and their counts can be pooled. We therefore define the pooled trial count and pooled success count for $f$ by
\[
N_f
\;:=\;
\sum_{(s,a,s')\in \mathrm{Occ}(f)} \#(s,a),
\qquad
K_f
\;:=\;
\sum_{(s,a,s')\in \mathrm{Occ}(f)} \#(s,a,s').
\]
Under the true instantiation $\uu$, the random variable $K_f$ is binomially distributed with parameters $N_f$ and $f[\uu]$. Hence, any binomial confidence-interval construction from Section~\ref{sec:learn-umdp} can be applied directly to the pooled counts $(K_f,N_f)$.
We distribute the confidence budget only across the non-constant expressions,
\[
\Lambda_{\mathsf{unk}}
\;:=\;
\{\, f\in\Lambda \mid f \notin \mathbb{Q}\,\},
\qquad
\gamma \;:=\; \frac{\delta}{|\Lambda_{\mathsf{unk}}|}.
\]
For each $f\in\Lambda_{\mathsf{unk}}$, we compute from a confidence interval $[l_f,u_f]$ such that
\[
\Pr\bigl[\, l_f \leq f[\uu] \leq u_f \,\bigr]
\;\ge\;
1-\gamma.
\]
For constant expressions $f\in\Lambda\setminus\Lambda_{\mathsf{unk}}$, we set $l_f=u_f=f$ exactly.
The resulting tied interval uncertainty set is obtained by assigning to each transition $(s,a,s')$ the interval associated with its expression:
\[
I(s,a,s')
\;:=\;
\bigl[l_{P_\Theta(s,a,s')},\,u_{P_\Theta(s,a,s')}\bigr].
\]
Equivalently, the parameter-tied IMDP is the uncertainty set
\[
\mathcal{P}_I(s,a)
\;:=\;
\Bigl\{\mu\in\Delta(S)\ \Big|\ \forall s'\in S:\ \mu(s') \in I(s,a,s')\Bigr\}.
\]
Since all occurrences of a fixed expression share the same pooled interval, parameter tying enforces exact equality between transitions labelled by the same expression.
Compared to the standard interval construction that learns a separate interval for every transition probability, parameter tying improves statistical efficiency in two ways. First, pooling increases the effective sample size from the local count $\#(s,a)$ to the aggregated count $N_f$, which typically yields narrower intervals. Second, the confidence budget is split across $|\Lambda_{\mathsf{unk}}|$ unknown expressions rather than across all non-constant transitions, which tightens the intervals further whenever the same expression occurs multiple times.

Parameter tying applies only to \emph{exact} equality of expressions. If two transitions depend on the same parameters but are labelled by different expressions, then their success probabilities need not coincide under the true instantiation, and the corresponding samples do not form repeated observations of the same Bernoulli event. In that case, pooling is not statistically justified. This is precisely the gap addressed by our parameter-space approach in Section~\ref{sec:learning}, which reasons jointly over all learned constraints instead of only over exactly matching expressions.

\section{Linear Relaxations for Polynomial Constraints}
\label{app:mccormick}
This appendix provides further details on the linear relaxation of polynomial constraints. The goal is to replace the nonlinear constraint system defining $\mathcal{U}$ by a polyhedral outer approximation that preserves soundness and can therefore be used within the rectangular relaxations of Section~\ref{sec:rectrelax}.

\vspace{10pt}
\noindent \textbf{Rewriting of monomials into bilinear terms.}
We view each polynomial expression $f\in\Lambda$ as a factorable expression built from constants, parameters, additions, and products. Since sums are already linear, only products require relaxation. Concretely, every monomial
\[
m(\theta) \;=\; c \prod_{j=1}^{k} \theta_{i_j}
\]
is rewritten by introducing auxiliary variables for intermediate products. We introduce auxiliary variables $z_1,\dots,z_{k-1}$ and enforce the chain
\[
z_1=\theta_{i_1}\theta_{i_2},
\qquad
z_r=z_{r-1}\theta_{i_{r+1}}
\quad\text{for } r=2,\dots,k-1,
\]
with $m(\theta)=c\,z_{k-1}$. Applying this construction to all monomials in all expressions yields an equivalent lifted representation of the nonlinear system in terms of linear equalities and bilinear equalities.

\vspace{10pt}
\noindent \textbf{McCormick relaxation of bilinear terms.}
Consider a bilinear equality $z=xy$ with bounds
$\underline{x}\le x\le \overline{x}$ and $\underline{y}\le y\le \overline{y}$.
Its McCormick relaxation is given by the four inequalities
\begin{equation}
\label{eq:mccormick-bilin-app}
\begin{aligned}
z &\ge \underline{y}\,x + \underline{x}\,y - \underline{x}\,\underline{y},
\qquad
z \le \underline{y}\,x + \overline{x}\,y - \overline{x}\,\underline{y},
\\
z &\ge \overline{y}\,x + \overline{x}\,y - \overline{x}\,\overline{y},
\qquad
z \le \overline{y}\,x + \underline{x}\,y - \underline{x}\,\overline{y}.
\end{aligned}
\end{equation}
Over the box $[\underline{x},\overline{x}] \times [\underline{y},\overline{y}]$, these inequalities describe the convex hull of the graph of $z=xy$~\cite{DBLP:journals/mp/McCormick76}. Hence, replacing every bilinear equality in the lifted system by~\eqref{eq:mccormick-bilin-app} yields a linear outer approximation of the original polynomial constraints.

\vspace{10pt}
\noindent \textbf{Global linear relaxation.}
Let $x$ collect all base parameters and all auxiliary variables introduced for intermediate products. The original nonlinear uncertainty region is defined by the constraints
\[
l_f \le f[\vv] \le u_f
\qquad\text{for all } f\in\Lambda,
\]
together with the box constraints $\vv\in\mathcal{D}$. After lifting and replacing each bilinear equality by its McCormick envelope, we obtain a polyhedron
\begin{equation}
\label{eq:global-mccormick-relaxation}
\mathcal{R}
\;=\;
\{\,x \mid A x \le b\,\},
\end{equation}
where $Ax\le b$ contains:
(i) the original linear constraints,
(ii) all McCormick inequalities for bilinear terms,
and
(iii) the variable bounds used in these inequalities.
By construction, every feasible point of the original nonlinear system extends to a feasible point of $\mathcal{R}$. Hence $\mathcal{R}$ is an outer approximation of the true feasible region, and any robust lower bound computed over $\mathcal{R}$ remains sound.

\vspace{10pt}
\noindent \textbf{Optimisation-based bound tightening.}
The quality of the McCormick relaxation depends crucially on the bounds used in~\eqref{eq:mccormick-bilin-app}. We therefore apply optimisation-based bound tightening (OBBT), which iteratively improves these bounds by solving LPs over the current relaxation. Concretely, for every variable $w$ in the lifted system, including both base parameters and auxiliary variables, we compute
\[
\underline{w}^{\,*} \;=\; \min \{\, w \mid Ax \le b \,\},
\qquad
\overline{w}^{\,*} \;=\; \max \{\, w \mid Ax \le b \,\}.
\]
These optimised bounds replace the previous interval bounds of $w$, after which all McCormick inequalities involving $w$ are rebuilt. This process is repeated until no bound improves beyond a prescribed tolerance or a maximum number of rounds.

Since tighter variable bounds yield tighter McCormick envelopes, OBBT can significantly strengthen the final linear relaxation. Soundness is preserved throughout, since each OBBT step optimises over the current outer approximation and therefore only removes values that are infeasible in that approximation.

\vspace{10pt}
\noindent \textbf{Remarks on tightness.}
McCormick relaxations are exact for a single bilinear term on a box, but for general factorable polynomials the resulting polyhedral relaxation depends on the chosen factorisation and on the variable bounds. In particular, different recursive decompositions of multilinear monomials can lead to different relaxations, and tighter bounds from OBBT typically improve the quality of the final relaxation substantially. This is well known in deterministic global optimisation and is one of the main motivations for combining McCormick relaxations with OBBT in practice~\cite{DBLP:journals/mp/McCormick76,DBLP:journals/jgo/ScottSB11,DBLP:journals/jgo/GleixnerBMW17}.

In our implementation, the construction of $\mathcal{R}$ and the OBBT procedure are performed once as a preprocessing step. The resulting polyhedral feasible region is then reused across all states and throughout robust value iteration, exactly as in the linear case considered in Section~\ref{sec:rectrelax}.

\section{Ellipsoidal Confidence Sets in Parameter Space}
\label{app:ellipsoidal}
For completeness, we also include an ellipsoidal baseline in parameter space.
This baseline is obtained by instantiating, in our pMDP setting, the
self-normalised least-squares confidence construction of
Abbasi-Yadkori et al.~\cite{DBLP:conf/nips/Abbasi-YadkoriPS11},
together with the linear-MDP adaptation discussed by
Ayoub et al.~(Appendix~C)~\cite{ayoub2020model}.
In contrast to the polytopic uncertainty sets used in the main body of the paper,
this yields an ellipsoidal confidence region over parameter
instantiations and therefore leads to conic optimisation in the robust Bellman
updates. 

Let $\mathcal{M}_{\Theta}=(S,A,s_0,P)$ be a pMDP with parameter space
$D \subseteq \R^\ell$, and let $\Lambda$ denote the set of distinct parametric
transition expressions occurring in $P$.
For each $(s,a,s') \in S \times A \times S$, we write
\[
P(s,a,s') = f_{s,a,s'} \in \Lambda.
\]
When all transition expressions are affine, each $f \in \Lambda$ can be written as
\[
f[\uu] = c_f + b_f^\top \uu,
\]
for some constant $c_f \in \R$, coefficients $b_f \in \R^\ell$, and
parameter instantiation~$\uu \in D$.

Ayoub et al.~\cite{ayoub2020model} derive confidence ellipsoids for finite-horizon linear MDPs by
combining least-squares regression with the self-normalised concentration bound
of Abbasi-Yadkori et al.~\cite{DBLP:conf/nips/Abbasi-YadkoriPS11}.
Their regression target is the overall value itself and tailored to episodic, finite horizon MDPs. In our case, the unknown object is the transition kernel induced by the
unknown parameter instantiation $\uu$, for infinite-horizon objectives. To extend this to our setting, we apply
the same concentration mechanism directly to the obtained one-step transition observations.

Concretely, for any fixed transition triple $(s,a,s')$, each visit to $(s,a)$
in the set of observed transitions $C$ contributes a binary observation
$Y \in \{0,1\}$ indicating whether the next state equals $s'$.
Its mean is precisely the transition probability $P[\uu](s,a,s')$.
Hence the centred noise term
\[
\eta := Y - P[\uu](s,a,s')
\]
is supported on an interval of length $1$, and is therefore
$\tfrac12$-sub-Gaussian by Hoeffding's lemma~\cite{DBLP:books/daglib/boucheron13}.
This is precisely the bounded-noise condition required by the
self-normalised least-squares bounds of
Abbasi-Yadkori et al.~\cite{DBLP:conf/nips/Abbasi-YadkoriPS11}.

Collecting these observations over the sample set $C$ yields, for each
expression $f \in \Lambda$, a family of Bernoulli trials with common mean
$f[\uu] = c_f + b_f^\top \uu$.
Since all transition triples in $\text{Occ}(f)$ are governed by the same symbolic
expression, they provide information about the same unknown linear quantity and
can therefore be pooled.
This gives a single count-based regression equation per distinct expression
rather than per concrete transition triple, which is exactly the least-squares
analogue of the parameter-tying mechanism used in our interval-based
constructions.
Aggregating these pooled equations over all $f \in \Lambda$ then yields the
following ridge least-squares system.

\[
V
   = \lambda I + \sum_{f \in \Lambda} N_f\, b_f b_f^\top,
\qquad
r
   = \sum_{f \in \Lambda} b_f \bigl(K_f - N_f c_f\bigr),
\]
with estimator
\[
\hat{\uu} = V^{-1} r.
\]
Here, $V$ is the regularised design matrix and $\lambda>0$ is the ridge
parameter.
The self-normalised concentration result of
Abbasi-Yadkori et al.~\cite{DBLP:conf/nips/Abbasi-YadkoriPS11}, instantiated in
the present count-based setting, yields the following sound confidence ellipsoid
\[
\mathcal{E}_\delta
   :=
   \left\{
      \vv \in \mathbb{R}^{\ell}
      \mid
      (\vv-\hat{\uu})^\top V (\vv-\hat{\uu}) \le \beta^2
   \right\},
\]
where
\[
\beta
   =
   R \sqrt{\log \frac{\det(V)}{\lambda^\ell}
           + 2 \log \frac{1}{\delta}}
   + \sqrt{\lambda}\, S,
\]
where $1 - \delta \in (0,1)$ is the overall confidence. With this construction, the unknown instantiation $\uu$ lies in $\mathcal{E}_\delta$ with probability at least
$1-\delta$ over the random sample set $C$~\cite{ayoub2020model}.
The constant $R$ is the sub-Gaussian noise bound, which for Bernoulli transition
observations is $R=\tfrac12$.
The constant $S$ is any a priori bound on $\|\uu\|_2$, for
$D \subseteq [0,1]^\ell$, the choice $S=\sqrt{\ell}$ is always sound, and $\lambda > 0$ is an arbitrary ridge parameter, which we choose as $\lambda = 10^{-2}$ to stabilise the regression. 

The ellipsoidal baseline induces, for each $(s,a)$, the uncertainty set
\[
\mathcal{P}_{R(\mathcal{E})}(s,a)
   :=
   \left\{
      \mu \in \Delta(S)
      \;\middle|\;
      \exists \vv \in \mathcal{E}_{\delta}
      :
      \mu(s') = P[\vv](s,a,s')
      \text{ for all } s' \in S
   \right\}.
\]
The inner optimisation in the robust Bellman update in Eq.~\eqref{eq:robust_bellman_rect} then becomes:
\[
\min_{\vv \in \mathcal{E}_{\delta}}
\sum_{s' \in S} P[\vv](s,a,s')\, V^*_U(s').
\]
Thus, robust synthesis under the ellipsoidal baseline requires solving a
convex conic optimisation problem for each Bellman backup.

As in the polytopic case, the relaxations of Section~\ref{sec:rectrelax} can also be applied to ellipsoidal uncertainty sets in order to obtain interval MDPs and thereby enable efficient robust synthesis. The difference lies in the optimisation problems used to construct these relaxations: instead of solving linear programmes for each expression in the expression-wise projection, or for each parameter in the parameter-wise projection, we solve second-order cone programmes over the ellipsoidal uncertainty region. The resulting interval model is again a sound rectangular over-approximation of the original ellipsoidal uncertainty set, and the inclusion hierarchy of Theorem~\ref{thm:inclusion} carries over accordingly.

Overall, this ellipsoidal construction is a natural adaption of the self-normalised least-squares confidence sets~\cite{ayoub2020model} to the pMDP setting. In the extended experimental evaluation in Appendix~\ref{app:experiments}, we compare this ellipsoidal baseline with the polytopic confidence regions and their corresponding relaxations across the considered case studies.

\section{Extended Experimental Evaluation}
\label{app:experiments}

In this appendix, we provide detailed descriptions of the considered case studies, as well as extended results complementing those in Section~\ref{sec:experiments}. In particular, we report a refined breakdown of runtime into time spent constructing and solving the respective uncertain models, and include additional results for further uncertainty classes, namely ellipsoids, as an additional baseline for comparison~\cite{zhou2021nearly,ayoub2020model}.

\subsection{Descriptions of Benchmark Environments}

\vspace{8pt}
\noindent \textbf{Aircraft collision avoidance.}
This benchmark models two aircraft moving towards each other on a discrete grid of size $N\times M$, where our agent controls one aircraft and aims to reach the far end of the grid without entering a collision zone around the adversarial aircraft~\cite{kochenderfer2015}. In each step, the agent chooses to ascend, descend, or maintain altitude while moving forward, and the adversarial aircraft simultaneously moves forward and may also change altitude. The transition probabilities are governed by four unknown mixture parameters $\Theta=\{\theta_1,\theta_2,\theta_3,\theta_4\}$, which combine four position-dependent motion modes. Concretely, the parameters determine both the success probability of the agent's vertical manoeuvre and the distribution of the adversary's vertical movement, with all probabilities varying affinely with the horizontal position along the grid. The resulting model therefore has a linear parametric structure with four dependent mixture weights. Our objective is to maximise the probability of reaching the goal without collision. We consider instances of sizes $(N,M)\in\{(50,10),(100,20),(200,40)\}$.

\vspace{8pt}
\noindent \textbf{Betting game.}
This benchmark models a sequential betting process in which the agent starts with 10 coins and, over a horizon of $n$ rounds, repeatedly chooses one of five actions corresponding to betting 0, 1, 2, 5, or 10 coins~\cite{DBLP:journals/mmor/BauerleO11}. After each non-zero bet, the agent either wins the wager and gains the bet amount, or loses it and forfeits the same amount. The transition probabilities are governed by two unknown parameters: a base win probability $\theta$ and a state-dependent term proportional to the current amount of money. In particular, for larger bets, the win probability is an affine function of both $\theta$ and the current capital, which yields a linear parametric structure with two parameters. The objective is to maximise the expected amount of money after $n$ betting rounds. We consider instances with horizons $n\in\{50,100,150\}$.

\vspace{8pt}
\noindent \textbf{Engagement.}
This benchmark models a customer-engagement process on a ladder of levels $\{0,\dots,L\}$, where level $0$ represents churn and level $L$ a successful conversion or purchase. At each decision step, the agent chooses one of three interventions, i.e., \emph{light}, \emph{medium}, or \emph{aggressive}, which may increase, decrease, or leave unchanged the current engagement level. The transition probabilities are governed by five unknown mixture parameters $\Theta=\{\theta_1,\theta_2,\theta_3,\theta_4,\theta_5\}$, with $\theta_5 = 1-(\theta_1+\theta_2+\theta_3+\theta_4)$, which combine five latent customer-response types. The effect of an intervention further depends on the current engagement zone (cold, warm, or hot), as well as on the previously chosen action and a short cooldown memory that penalises repeated aggressive interventions. The resulting model therefore has a linear parametric structure with five dependent mixture weights and non-trivial state-dependent dynamics. Rewards combine a per-step time cost, action-dependent intervention costs, and a penalty for churn. The objective is to minimise the expected accumulated cost until either churn or purchase is reached. We consider instances with $L\in\{100,300,1000\}$.

\vspace{8pt}
\noindent \textbf{Mars rover.}
This benchmark is a richer variant of the semi-autonomous vehicle domain of~\cite{DBLP:conf/tacas/Junges0DTK16}, and is already introduced in Section~\ref{sec:intro}. A rover moves on an $X\times Y$ grid and aims to reach a target location while maintaining communication with a remote controller via two unreliable channels. At each step, it may either move in one of the allowed directions or attempt communication over one of the two channels, with task failure occurring if too many moves are made without a successful transmission. The transition probabilities are governed by six unknown mixture parameters $\Theta=\{\theta_1,\theta_2,\theta_3,\theta_4,\theta_5,\theta_6\}$, which combine five position-dependent channel modes. In contrast to the original benchmark, our variant includes a larger grid, a richer spatial structure with zone-dependent scaling and local jamming regions, and a larger communication budget, yielding a more varied but still linear parametric structure. We consider instances of sizes $(X,Y)\in\{(50,50),(75,75),(125,125)\}$.

\vspace{8pt}
\noindent \textbf{Glider.}
This benchmark models an autonomous underwater glider navigating on a two-dimensional grid from a start location to a goal region in the presence of ocean currents~\cite{DBLP:conf/aaai/CostenRLH23}. At each location, the horizontal and vertical current components are known, but their effect on the glider's movement is governed by two unknown parameters, $\theta_h$ and $\theta_v$, capturing the sensitivity to horizontal and vertical drift, respectively. When the glider chooses a movement direction, the commanded motion may fail along the intended axis and may additionally drift along the orthogonal axis, yielding up to four possible outcomes per action. Because these two effects combine multiplicatively, the transition probabilities contain bilinear terms in $\theta_h$ and $\theta_v$, giving rise to a non-linear parametric structure. The objective is to minimise the expected time to reach the goal. We consider instances of sizes $(X,Y)\in\{(21,17), (51,47),(106,99)\}$.

\vspace{8pt}
\noindent \textbf{Parallel betting game.}
This benchmark is a parallel composition of two betting games that are played simultaneously over a common horizon of $n$ rounds. In each round, the agent chooses a single betting action, corresponding to wagering 0, 1, 2, 5, or 10 coins, and this action is applied independently to both games. Each game maintains its own capital and is governed by separate unknown parameters, determining its base win probability. As in the single-game benchmark, larger bets have outcome probabilities that additionally depend on the current amount of money. Since the overall state combines the capitals of both games and the final reward is their sum, the resulting model captures a coupled decision problem with two unknown parameters. Moreover, because the transition probabilities in the second game depend multiplicatively on $\theta_2$ and the current capital, the parallel composition yields a non-linear parametric structure. The objective is to maximise the expected total amount of money after $n$ rounds. We consider instances with horizons $n\in\{5,10,20\}$.

\subsection{Extended Results for Uniform Sampling}

Table~\ref{tab:results_udtmc_optpolicy_interval_relgap_split_runtime} reports the extended results for the first setting, in which $10^5$ trajectories are sampled uniformly and used to compare the tightness of the resulting uncertain models. In addition to the overall results, we separate the total runtime into the time spent constructing the uncertain model and the time spent solving it.

\begin{landscape}
\begin{table}[t]
\centering
\caption{Comparison of obtained bounds and runtimes for both building and solving the respective models obtained after sampling $10^5$ trajectories uniformly.}
\vspace{5pt}
\setlength{\tabcolsep}{0.5pt}
\small
\resizebox{1.58\textwidth}{!}{%
\begin{tabular}{>{\centering\arraybackslash}m{2.35cm} >{\centering\arraybackslash}m{1.65cm} >{\centering\arraybackslash}m{2.0cm} @{\hspace{5pt}} >{\centering\arraybackslash}m{0.95cm} @{\hspace{5pt}} >{\centering\arraybackslash}m{0.95cm} @{\hspace{5pt}} >{\centering\arraybackslash}m{0.95cm} @{\hspace{9pt}} >{\centering\arraybackslash}m{2.0cm} @{\hspace{5pt}} >{\centering\arraybackslash}m{0.95cm} @{\hspace{5pt}} >{\centering\arraybackslash}m{0.95cm} @{\hspace{5pt}} >{\centering\arraybackslash}m{0.95cm} @{\hspace{9pt}} >{\centering\arraybackslash}m{2.0cm} @{\hspace{5pt}} >{\centering\arraybackslash}m{0.95cm} @{\hspace{5pt}} >{\centering\arraybackslash}m{0.95cm} @{\hspace{5pt}} >{\centering\arraybackslash}m{0.95cm} @{\hspace{9pt}} >{\centering\arraybackslash}m{2.0cm} @{\hspace{5pt}} >{\centering\arraybackslash}m{0.95cm} @{\hspace{5pt}} >{\centering\arraybackslash}m{0.95cm} @{\hspace{5pt}} >{\centering\arraybackslash}m{0.95cm}}
\toprule
\multirow[c]{2}{*}{\textbf{Benchmark}} & \multirow[c]{2}{*}{\textbf{Instance}} & \multicolumn{4}{c}{$\pmb{\mathcal{P}_I}$} & \multicolumn{4}{c}{$\pmb{\mathcal{P}_{\Theta(\mathcal{U})}}$} & \multicolumn{4}{c}{$\pmb{\mathcal{P}_{\Lambda(\mathcal{U})}}$} & \multicolumn{4}{c}{$\pmb{\mathcal{P}_{R(\mathcal{U})}}$} \\
\cmidrule(lr){3-6} \cmidrule(lr){7-10} \cmidrule(lr){11-14} \cmidrule(lr){15-18}
 &  & $\mathbf{[\underline{V},\overline{V}]}$ & \textbf{Rel. gap} & \textbf{Build [s]} & \textbf{Solve [s]} & $\mathbf{[\underline{V},\overline{V}]}$ & \textbf{Rel. gap} & \textbf{Build [s]} & \textbf{Solve [s]} & $\mathbf{[\underline{V},\overline{V}]}$ & \textbf{Rel. gap} & \textbf{Build [s]} & \textbf{Solve [s]} & $\mathbf{[\underline{V},\overline{V}]}$ & \textbf{Rel. gap} & \textbf{Build [s]} & \textbf{Solve [s]} \\
\midrule
\multirow[c]{3}{*}{\shortstack[c]{Aircraft}} 
 & (50, 10) & $[0.67,\,0.814]$ & 0.19 & 0.77 & 3.28 & $[0.618,\,0.828]$ & 0.28 & 1.29 & 3.34 & $[0.719,\,0.767]$ & 0.06 & 2.17 & 3.55 & $[0.725,\,0.761]$ & \textbf{0.05} & 44.95 & 2.17 \\
 & (100, 20) & $[0.745,\,0.906]$ & 0.19 & 4.11 & 26.34 & $[0.679,\,0.918]$ & 0.29 & 8.77 & 25.94 & $[0.805,\,0.865]$ & 0.07 & 9.26 & 30.28 & $[0.814,\,0.858]$ & \textbf{0.05} & 309.67 & 43.69 \\
 & (200, 40) & $[0.628,\,0.92]$ & 0.36 & 43.23 & 645.05 & $[0.55,\,0.928]$ & 0.47 & 114.09 & 763.95 & $[0.754,\,0.85]$ & 0.12 & 128.50 & 754.28 & $[0.769,\,0.838]$ & \textbf{0.09} & 2175.51 & 2288.73 \\
\addlinespace[4pt]
\multirow[c]{3}{*}{\shortstack[c]{Betting Game}} & (50) & $[14.4,\,65.3]$ & 1.88 & 0.92 & 1.79 & $[25.8,\,28.7]$ & \textbf{0.10} & 1.57 & 4.92 & $[25.8,\,28.7]$ & \textbf{0.10} & 4.67 & 2.87 & $[25.8,\,28.7]$ & \textbf{0.10} & 1.28 & 1.01 \\
 & (100) & $[16.5,\,162]$ & 3.24 & 3.10 & 9.98 & $[41.8,\,47.7]$ & \textbf{0.13} & 4.99 & 7.87 & $[41.8,\,47.7]$ & \textbf{0.13} & 16.87 & 9.44 & $[41.8,\,47.7]$ & \textbf{0.13} & 8.79 & 8.26 \\
 & (150) & $[17.5,\,263]$ & 3.98 & 5.69 & 19.49 & $[56.5,\,65.5]$ & \textbf{0.15} & 19.36 & 18.61 & $[56.5,\,65.5]$ & \textbf{0.15} & 35.29 & 17.87 & $[56.5,\,65.5]$ & \textbf{0.15} & 13.10 & 25.24 \\
\addlinespace[4pt]
\multirow[c]{3}{*}{Engagement} & (100) & $[38.4,\,49.3]$ & 0.26 & 0.12 & 5.33 & $[16.6,\,5422]$ & 127.35 & 0.11 & 4.33 & $[39.8,\,45]$ & \textbf{0.12} & 0.19 & 0.82 & $[39.8,\,45]$ & \textbf{0.12} & 1.24 & 0.58 \\
 & (300) & $[38.6,\,46.8]$ & 0.19 & 0.32 & 16.23 & $[18.7,\,3615]$ & 84.74 & 0.31 & 20.98 & $[40.8,\,45.3]$ & \textbf{0.11} & 0.48 & 7.47 & $[40.8,\,45.3]$ & \textbf{0.11} & 1.41 & 5.86 \\
 & (1000) & $[39.1,\,45.4]$ & 0.15 & 0.98 & 69.04 & $[24,\,687]$ & 15.63 & 0.91 & 56.83 & $[41.3,\,43.8]$ & \textbf{0.06} & 1.26 & 64.54 & $[41.3,\,43.8]$ & \textbf{0.06} & 2.17 & 93.27 \\
\addlinespace[4pt]
\multirow[c]{3}{*}{\shortstack[c]{Mars Rover}} 
 & (50, 50) & $[0,\,0.797]$ & 1.59 & 9.38 & 82.64 & $[0.177,\,0.833]$ & 1.31 & 21.09 & 115.78 & $[0.488,\,0.515]$ & \textbf{0.05} & 27.74 & 87.41 & -- & -- & TO & TO \\
 & (75, 75) & $[0,\,0.895]$ & 1.42 & 15.96 & 505.71 & $[0.0904,\,0.9]$ & 1.28 & 47.80 & 529.12 & $[0.619,\,0.647]$ & \textbf{0.04} & 61.38 & 518.83 & -- & -- & TO & TO \\
 & (125, 125) & $[0,\,0.877]$ & 1.42 & 19.64 & 2149.57 & $[0.0441,\,0.922]$ & 1.42 & 97.87 & 2455.97 & $[0.599,\,0.635]$ & \textbf{0.06} & 111.17 & 2195.12 & -- & -- & TO & TO \\
 \addlinespace[4pt]
\multirow[c]{3}{*}{Glider} & (21, 17) & $[41.1,\,44.6]$ & 0.08 & 0.16 & 0.09 & $[42.4,\,42.9]$ & 0.01 & 0.16 & 0.03 & $[42.5,\,42.8]$ & 0.01 & 0.65 & 0.06 & $[42.5,\,42.8]$ & \textbf{0.01} & 37.20 & 0.06 \\
 & (51, 47) & $[62.3,\,20001]$ & 293.07 & 0.52 & 28.40 & $[58.8,\,59.2]$ & 0.01 & 1.12 & 0.55 & $[58.8,\,59.1]$ & 0.00 & 16.59 & 0.57 & $[67.9,\,68.1]$ & \textbf{0.00} & 9266.65 & 1.58 \\
 & (106, 99) & $[66.4,\,35254]$ & 464.04 & 0.52 & 52.43 & $[75.6,\,76.1]$ & 0.01 & 1.91 & 1.97 & $[75.6,\,76]$ & \textbf{0.00} & 46.57 & 1.94 & -- & -- & TO & TO \\
 \addlinespace[4pt]
\multirow[c]{3}{*}{\shortstack[c]{Parallel Betting}}
 & (5) & $[24.3,\,44.7]$ & 0.64 & 1.89 & 0.69 & $[25.1,\,28.2]$ & 0.10 & 0.38 & 0.02 & $[26.1,\,27.2]$ & 0.04 & 0.77 & 0.03 & $[26.1,\,27.2]$ & \textbf{0.03} & 10.13 & 0.03 \\
 & (10) & $[26.4,\,40.6]$ & 0.44 & 2.13 & 0.73 & $[29.7,\,34.5]$ & 0.15 & 2.83 & 0.40 & $[31.1,\,32.9]$ & 0.06 & 5.79 & 0.63 & $[31.1,\,32.8]$ & \textbf{0.05} & 116.11 & 0.41 \\
 & (20) & $[13,\,110]$ & 3.39 & 16.00 & 5.24 & $[25.4,\,31.7]$ & \textbf{0.22} & 659.37 & 6.94 & -- & -- & TO & TO & -- & -- & TO & TO \\
\bottomrule
\end{tabular}
} 
\label{tab:results_udtmc_optpolicy_interval_relgap_split_runtime}
\end{table}
\end{landscape}

\subsection{Extended Learning Results}

This section presents the full results for the online learning experiments. Table~\ref{tab:stats_same_suite_learning} summarises the salient characteristics of the benchmark instances used in this setting. Since the uncertain models must be rebuilt and resolved repeatedly in order to update the sampling policy, these instances are smaller than those used in the uniform-sampling setting, where each model is solved only once. Overall, the results show that exploiting the parametric structure consistently yields more effective policies and stronger certified performance guarantees from the same data and under the same high-confidence PAC guarantees.

\begin{table}[t]
\centering
\caption{Salient characteristics of the benchmarks for the online learning setting.}
\vspace{4pt}
\setlength{\tabcolsep}{2pt}
\footnotesize
\renewcommand{\arraystretch}{0.92}
\resizebox{0.9\textwidth}{!}{%
\begin{tabular}{>{\centering\arraybackslash}m{3.3cm} >{\centering\arraybackslash}m{1.8cm} >{\centering\arraybackslash}m{1.2cm} >{\centering\arraybackslash}m{1.35cm} >{\centering\arraybackslash}m{1.0cm} >{\centering\arraybackslash}m{1.0cm} >{\centering\arraybackslash}m{1.0cm}}
\toprule
\textbf{Benchmark} & \textbf{Instance} & \textbf{$|S|$} & \textbf{$|T|$} & \textbf{$|\Theta|$} & \textbf{$|\Lambda|$} & \textbf{Prop.} \\
\midrule
\multirow[c]{1}{*}{Aircraft} & (20, 10) & 1833 & 23743 & 6 & 271 & $\mathbb{P}$ \\
\addlinespace[2pt]
\multirow[c]{1}{*}{Betting Game} & (25, 0.55) & 3625 & 27703 & 2 & 1283 & $\mathbb{E}$ \\
\addlinespace[2pt]
\multirow[c]{1}{*}{Engagement} & (50, 0.30) & 994 & 2962 & 5 & 61 & $\mathbb{E}$ \\
\addlinespace[2pt]
\multirow[c]{1}{*}{Mars Rover} & (10, 10) & 5139 & 24515 & 6 & 145 & $\mathbb{P}$ \\
\addlinespace[2pt]
\multirow[c]{1}{*}{Glider} & (11, 9) & 99 & 1317 & 4 & 225 & $\mathbb{E}$ \\
\addlinespace[2pt]
\multirow[c]{1}{*}{Parallel Betting} & (6, 0.55) & 5046 & 30746 & 3 & 597 & $\mathbb{E}$ \\
\bottomrule
\end{tabular}
} 
\label{tab:stats_same_suite_learning}
\end{table}
\renewcommand{\arraystretch}{1}

\begin{center}
  \includegraphics[width=0.95\textwidth]{figures/learning/legend_border.pdf}
  \vspace{1em}


  \twoplotblock
{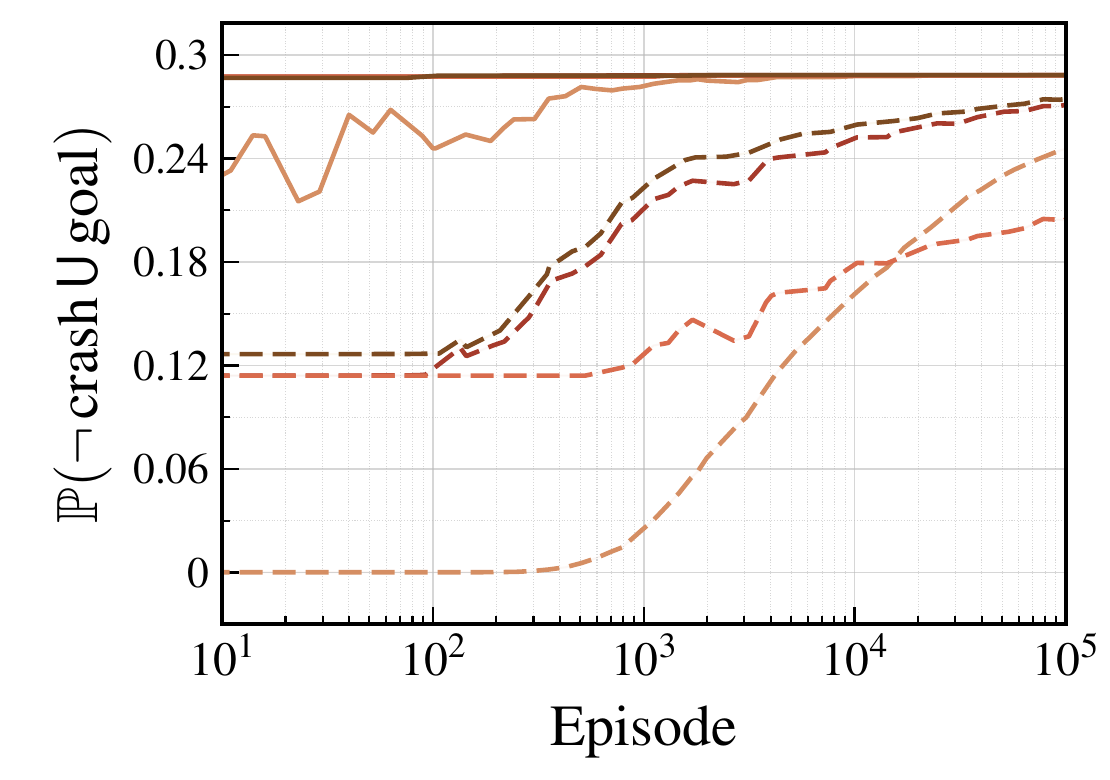}
{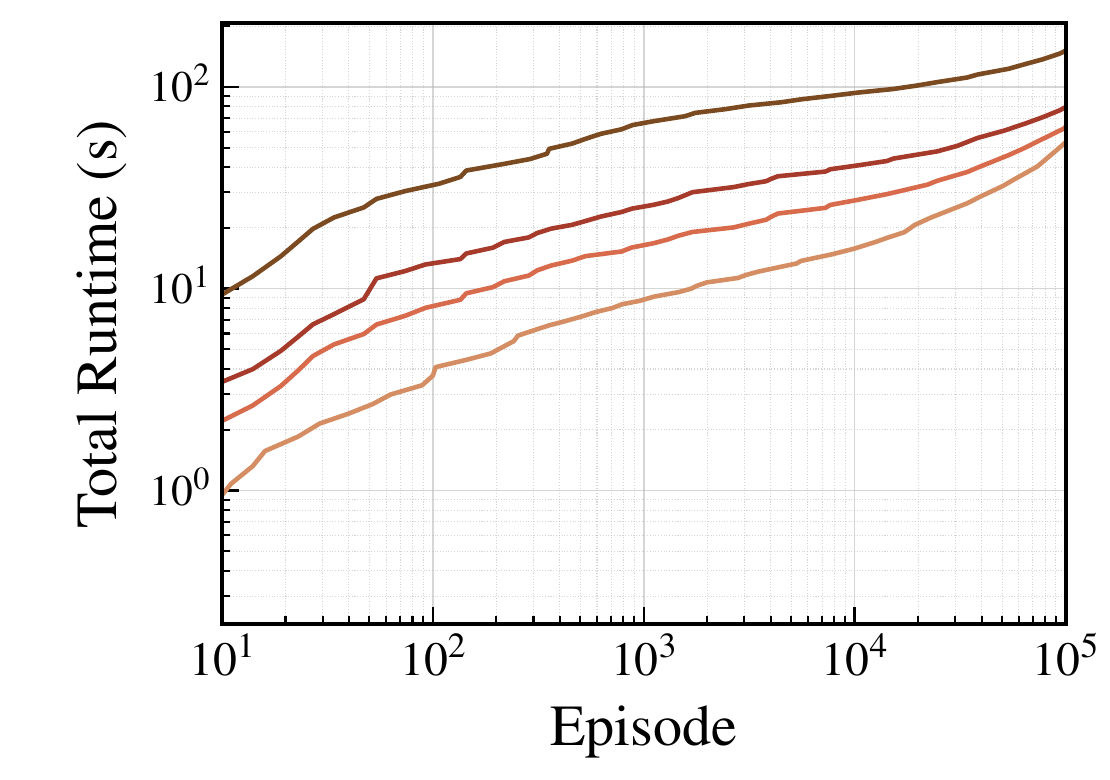}
    {Aircraft}
    {fig:aircraft}
  \hspace{2pt}
  \twoplotblock
{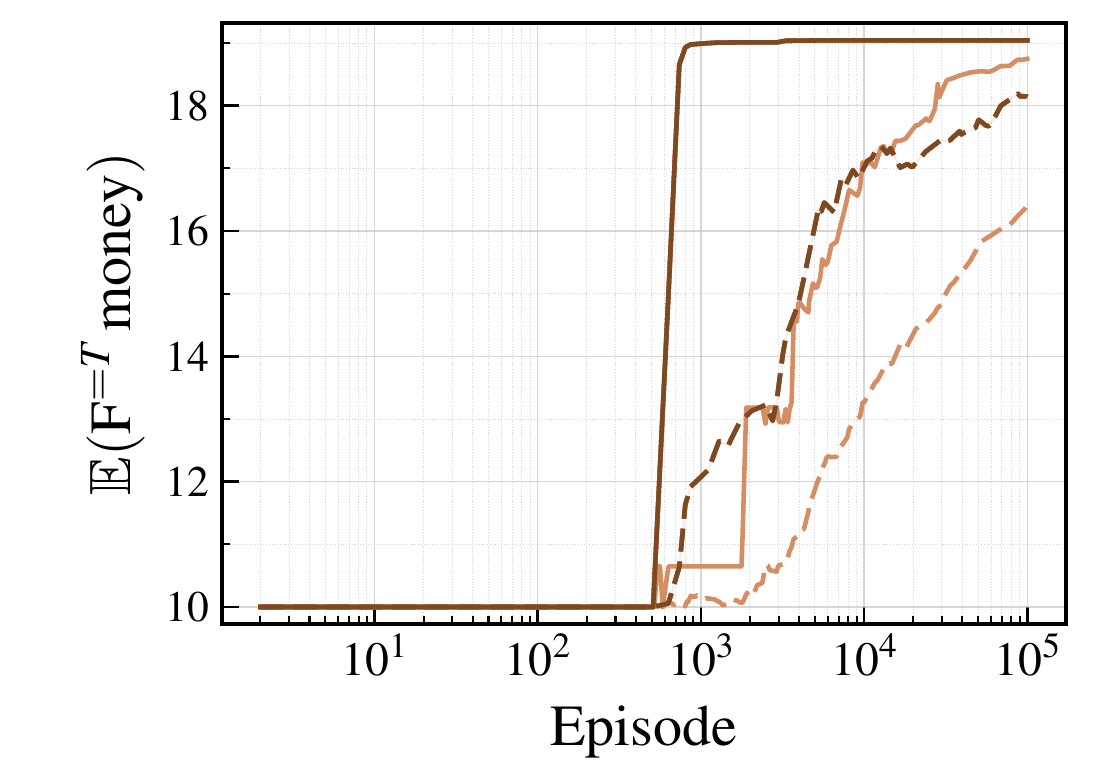}{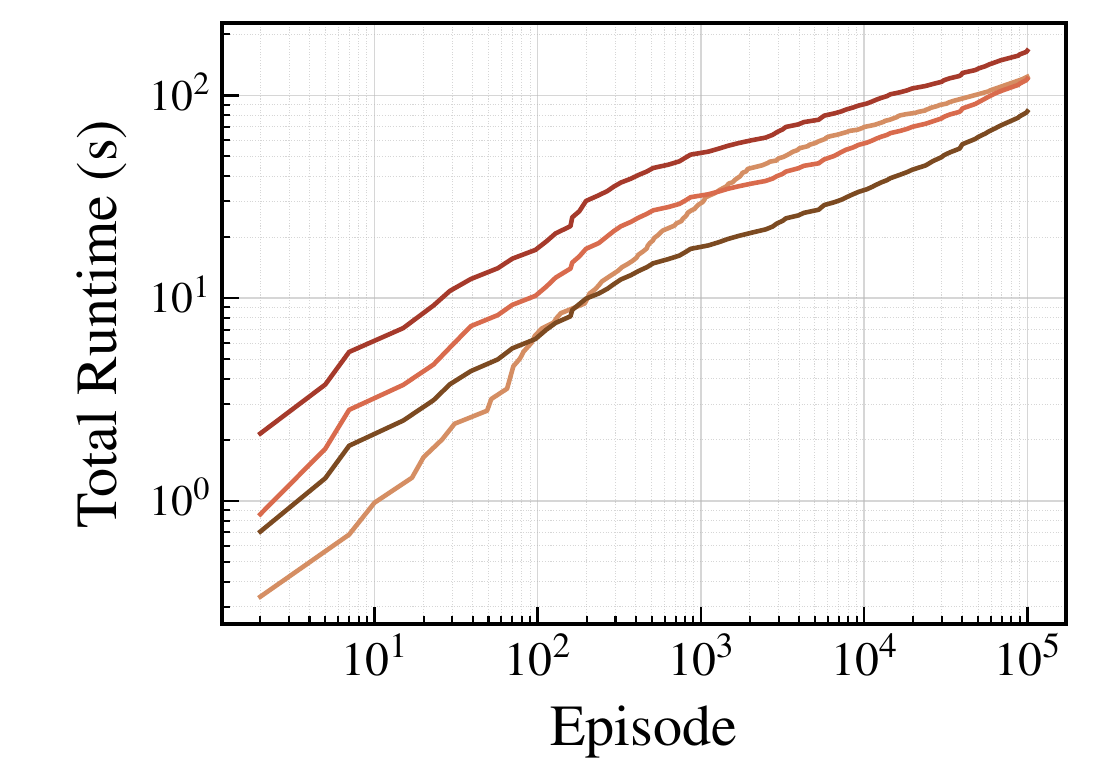}
    {Betting Game}
    {fig:betting}

  \vspace{1.1em}
  \twoplotblock{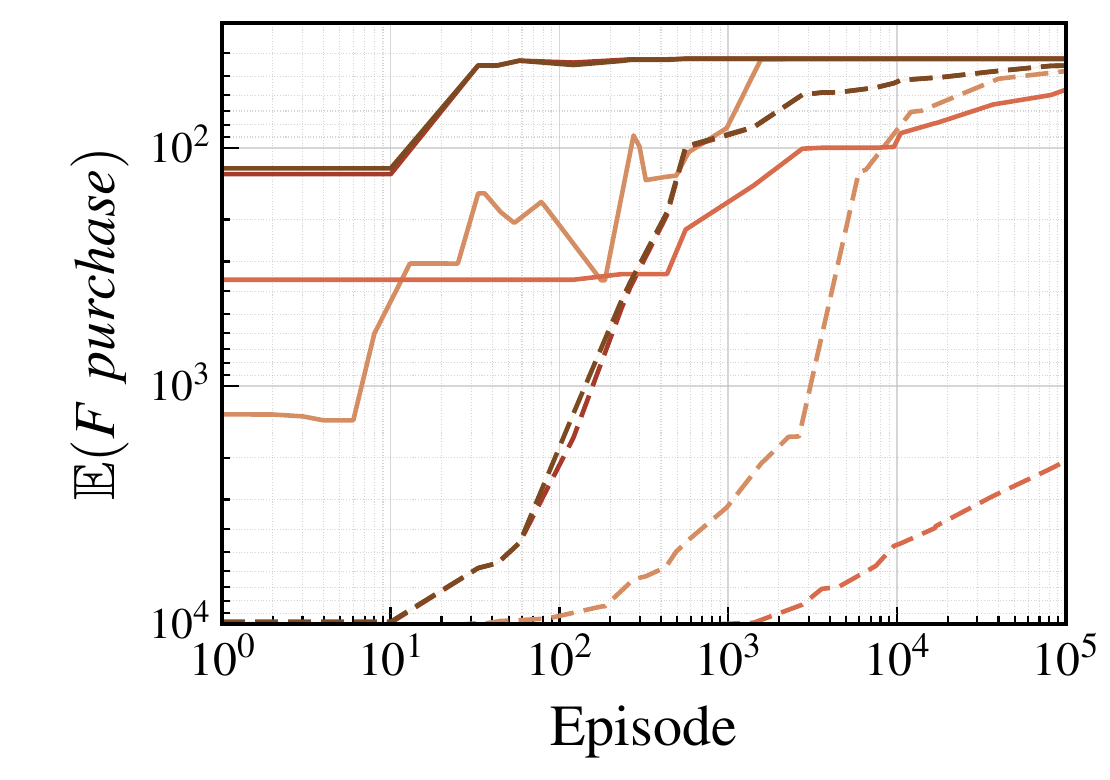}{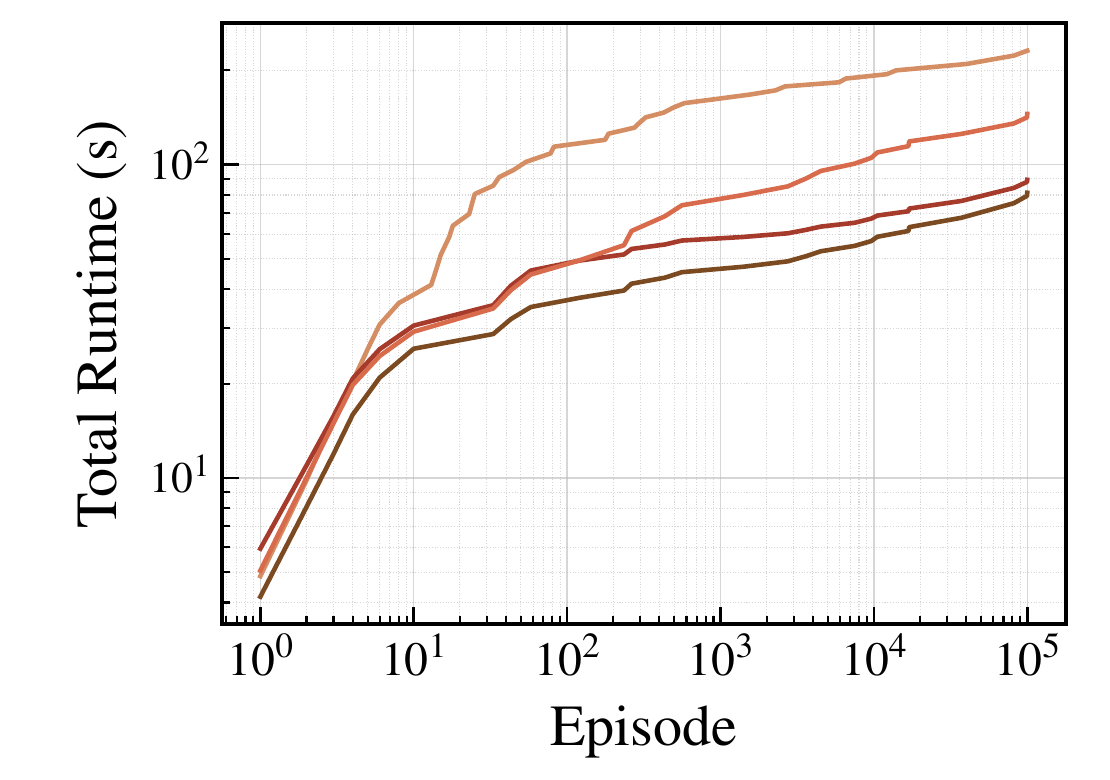}
    {Engagement}
    {fig:engagement}
  \hspace{2pt}
\twoplotblock{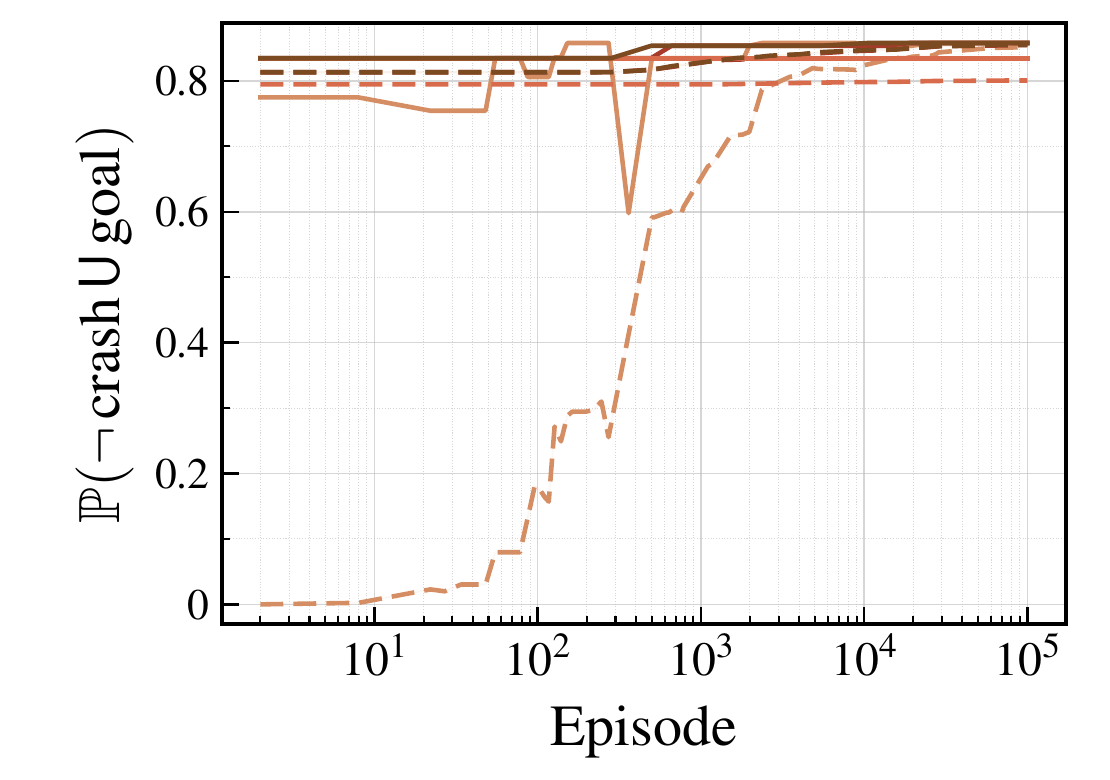}{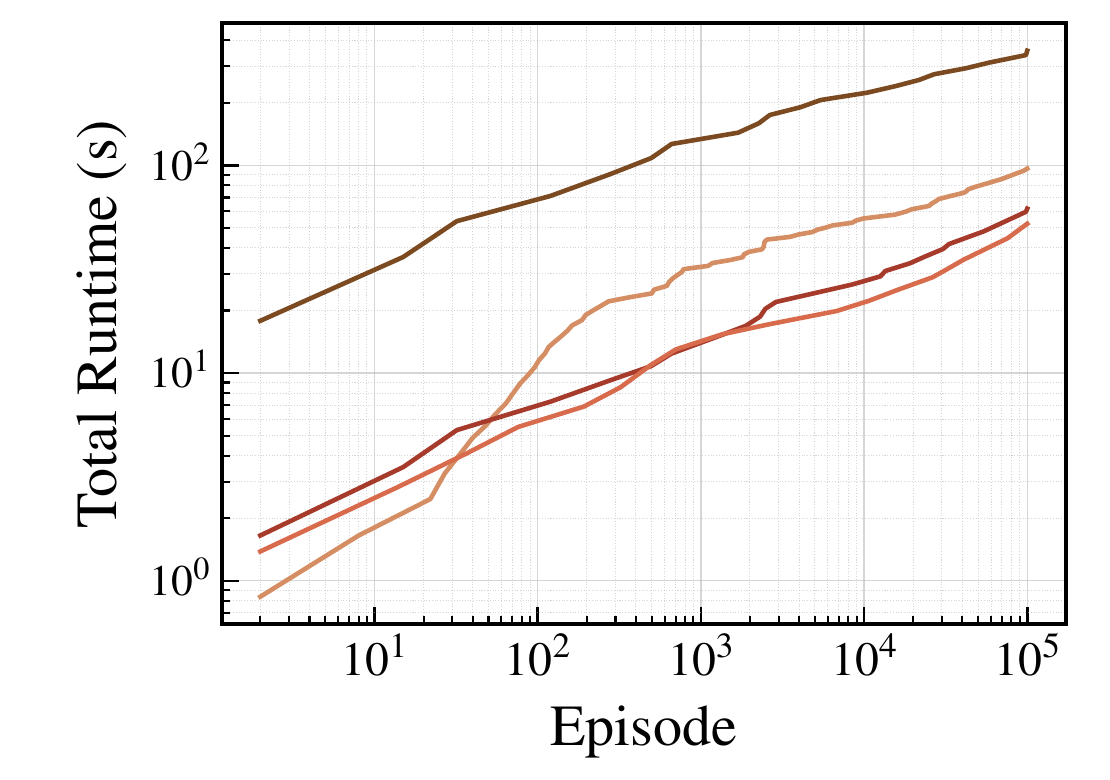}
    {Mars Rover}
    {fig:marsrover}

  \vspace{1.1em}
  \twoplotblock{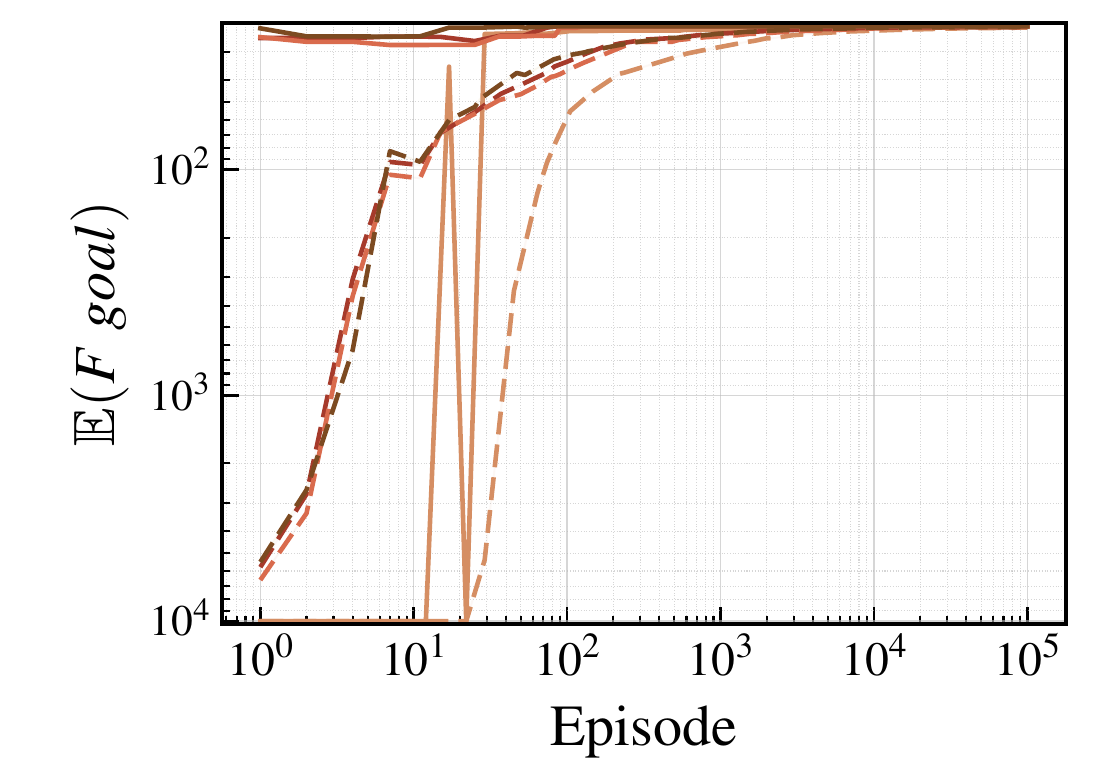}{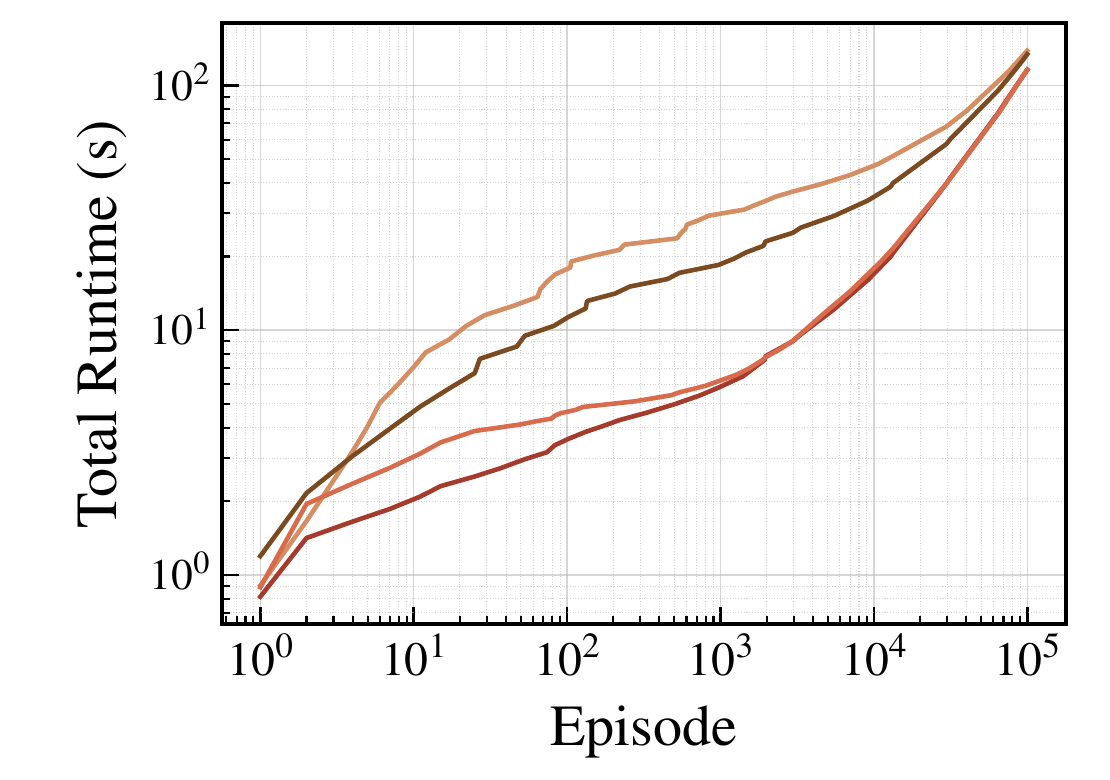}
    {Glider}
    {fig:glider}
  \hspace{2pt}
\twoplotblock{figures/plots/BETTING_GAME_PARALLEL/n=6,p_1=0.55,p_2=0.53/BETTING_GAME_PARALLEL_n=6,p_1=0.55,p_2=0.53.pdf}{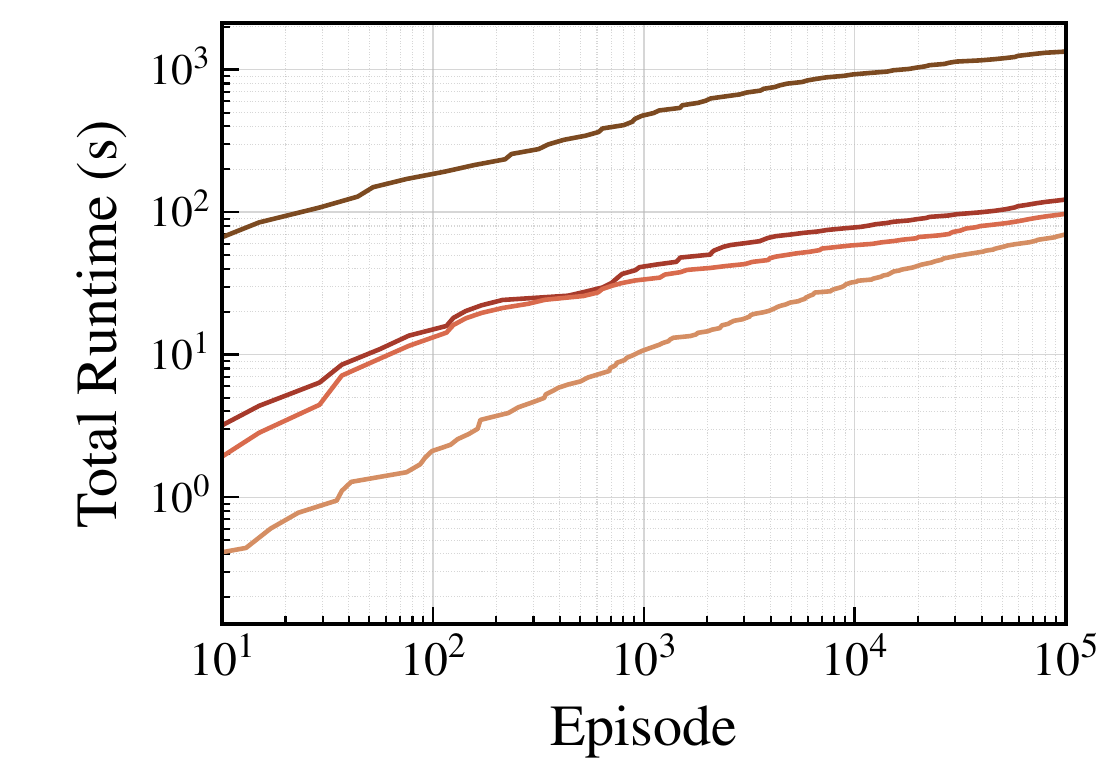}
    {Parallel Betting}
    {fig:bettingparallel}

      \captionof{figure}{Extended results for robust policy learning. For each benchmark, the upper plots compare performance and guarantees, and the lower plots show the total runtimes of each method.}
  \label{fig:learning_extended}
\end{center}

\subsection{Experimental Results with Ellipsoidal Uncertainty Sets}

This section reports the results obtained when instantiating our learning and solving framework with the ellipsoidal uncertainty sets discussed in Appendix~\ref{app:ellipsoidal}. Table~\ref{tab:bigtable} summarises the results for the uniform-sampling setting, comparing runtimes and the resulting certified guarantees across the different uncertainty sets and relaxations.
We omit the local rectangular relaxation $\mathcal{P}_{R(\mathcal{E})}$ from the table, as it timed out on all considered case studies. This is due to the high computational cost of solving a second-order cone programme for every state--action pair in every iteration of robust value iteration. This highlights the importance of the interval-based relaxations, which make ellipsoidal uncertainty practically usable by projecting it back to tractable interval uncertainty sets.

Overall, the results indicate that the approach based on projecting confidence intervals into the parameter space consistently yields tighter uncertainty sets and stronger guarantees than the ellipsoidal baselines. 

\begin{landscape}
\begin{table}[t]
\centering
\caption{Comparison of obtained bounds and runtimes for both building and solving the respective models obtained after sampling $10^5$ trajectories uniformly. This includes both polytopic and ellipsoidal uncertainty sets and their relaxations.}
\vspace{5pt}
\setlength{\tabcolsep}{0.5pt}
\small
\resizebox{1.58\textwidth}{!}{%
\def\PairSpacerColor{white}
\begin{tabular}{>{\centering\arraybackslash}m{2.35cm} >{\centering\arraybackslash}m{1.65cm} >{\centering\arraybackslash}m{2.0cm} @{\hspace{1.8pt}} >{\centering\arraybackslash}m{0.95cm} @{\hspace{1.8pt}} >{\centering\arraybackslash}m{0.95cm} @{\hspace{1.8pt}} >{\centering\arraybackslash}m{0.95cm} @{\hspace{3.2pt}} >{\columncolor{gray!10}\centering\arraybackslash}m{2.22cm} @{\color{\PairSpacerColor}\vrule width 1.8pt} >{\columncolor{gray!10}\centering\arraybackslash}m{1.20cm} @{\color{\PairSpacerColor}\vrule width 1.8pt} >{\columncolor{gray!10}\centering\arraybackslash}m{1.20cm} @{\color{\PairSpacerColor}\vrule width 1.8pt} >{\columncolor{gray!10}\centering\arraybackslash}m{1.20cm} @{\color{\PairSpacerColor}\vrule width 3.2pt} >{\columncolor{gray!10}\centering\arraybackslash}m{2.22cm} @{\color{\PairSpacerColor}\vrule width 1.8pt} >{\columncolor{gray!10}\centering\arraybackslash}m{1.20cm} @{\color{\PairSpacerColor}\vrule width 1.8pt} >{\columncolor{gray!10}\centering\arraybackslash}m{1.20cm} @{\color{\PairSpacerColor}\vrule width 1.8pt} >{\columncolor{gray!10}\centering\arraybackslash}m{1.20cm} @{\hspace{3.2pt}} >{\centering\arraybackslash}m{2.0cm} @{\hspace{1.8pt}} >{\centering\arraybackslash}m{0.95cm} @{\hspace{1.8pt}} >{\centering\arraybackslash}m{0.95cm} @{\hspace{1.8pt}} >{\centering\arraybackslash}m{0.95cm} @{\hspace{3.2pt}} >{\centering\arraybackslash}m{2.0cm} @{\hspace{1.8pt}} >{\centering\arraybackslash}m{0.95cm} @{\hspace{1.8pt}} >{\centering\arraybackslash}m{0.95cm} @{\hspace{1.8pt}} >{\centering\arraybackslash}m{0.95cm} @{\hspace{3.2pt}} >{\columncolor{gray!10}\centering\arraybackslash}m{2.22cm} @{\color{\PairSpacerColor}\vrule width 1.8pt} >{\columncolor{gray!10}\centering\arraybackslash}m{1.20cm} @{\color{\PairSpacerColor}\vrule width 1.8pt} >{\columncolor{gray!10}\centering\arraybackslash}m{1.20cm} @{\color{\PairSpacerColor}\vrule width 1.8pt} >{\columncolor{gray!10}\centering\arraybackslash}m{1.20cm}}
\toprule
\multirow[c]{2}{*}{\textbf{Benchmark}} & \multirow[c]{2}{*}{\textbf{Instance}} & \multicolumn{4}{c}{$\pmb{\mathcal{P}_I}$} & \multicolumn{4}{c}{$\pmb{\mathcal{P}_{\Theta(\mathcal{U})}}$} & \multicolumn{4}{c}{$\pmb{\mathcal{P}_{\Theta(\mathcal{E})}}$} & \multicolumn{4}{c}{$\pmb{\mathcal{P}_{\Lambda(\mathcal{U})}}$} & \multicolumn{4}{c}{$\pmb{\mathcal{P}_{\Lambda(\mathcal{E})}}$} & \multicolumn{4}{c}{$\pmb{\mathcal{P}_{R(\mathcal{U})}}$} \\
\cmidrule(lr){3-6} \cmidrule(lr){7-10} \cmidrule(lr){11-14} \cmidrule(lr){15-18} \cmidrule(lr){19-22} \cmidrule(lr){23-26}
 &  & $\mathbf{[\underline{V},\overline{V}]}$ & \textbf{Rel. gap} & \textbf{Build [s]} & \textbf{Solve [s]} & \cellcolor{white}$\mathbf{[\underline{V},\overline{V}]}$ & \cellcolor{white}\textbf{Rel. gap} & \cellcolor{white}\textbf{Build [s]} & \cellcolor{white}\textbf{Solve [s]} & \cellcolor{white}$\mathbf{[\underline{V},\overline{V}]}$ & \cellcolor{white}\textbf{Rel. gap} & \cellcolor{white}\textbf{Build [s]} & \cellcolor{white}\textbf{Solve [s]} & $\mathbf{[\underline{V},\overline{V}]}$ & \textbf{Rel. gap} & \textbf{Build [s]} & \textbf{Solve [s]} & $\mathbf{[\underline{V},\overline{V}]}$ & \textbf{Rel. gap} & \textbf{Build [s]} & \textbf{Solve [s]} & \cellcolor{white}$\mathbf{[\underline{V},\overline{V}]}$ & \cellcolor{white}\textbf{Rel. gap} & \cellcolor{white}\textbf{Build [s]} & \cellcolor{white}\textbf{Solve [s]} \\
\midrule
\global\def\PairSpacerColor{gray!10}\multirow[c]{3}{*}{Aircraft} & (50, 10) & $[0.67,\,0.814]$ & 0.19 & 0.30 & 0.43 & $[0.618,\,0.828]$ & 0.28 & 0.95 & 0.60 & $[0.599,\,0.835]$ & 0.32 & 0.50 & 0.50 & $[0.719,\,0.767]$ & 0.06 & 1.13 & 0.60 & $[0.71,\,0.78]$ & 0.09 & 1.86 & 0.63 & $[0.725,\,0.761]$ & \textbf{0.05} & 28.02 & 0.54 \\
 & (100, 20) & $[0.745,\,0.906]$ & 0.19 & 2.02 & 16.54 & $[0.679,\,0.918]$ & 0.29 & 6.56 & 27.81 & $[0.678,\,0.919]$ & 0.29 & 3.14 & 14.67 & $[0.805,\,0.865]$ & 0.07 & 7.47 & 19.30 & $[0.805,\,0.866]$ & 0.07 & 28.12 & 28.47 & $[0.814,\,0.858]$ & \textbf{0.05} & 217.95 & 20.55 \\
 & (200, 40) & $[0.628,\,0.92]$ & 0.36 & 15.71 & 282.38 & $[0.55,\,0.928]$ & 0.47 & 54.06 & 327.19 & $[0.575,\,0.926]$ & 0.44 & 22.96 & 266.26 & $[0.754,\,0.85]$ & 0.12 & 61.25 & 253.13 & $[0.768,\,0.844]$ & 0.09 & 1667.3 & 344.75 & $[0.769,\,0.838]$ & \textbf{0.09} & 1438.33 & 754.25 \\
\rule{0pt}{4pt} & \rule{0pt}{4pt} & \rule{0pt}{4pt} & \rule{0pt}{4pt} & \rule{0pt}{4pt} & \rule{0pt}{4pt} & \cellcolor{gray!10}\rule{0pt}{4pt} & \cellcolor{gray!10}\rule{0pt}{4pt} & \cellcolor{gray!10}\rule{0pt}{4pt} & \cellcolor{gray!10}\rule{0pt}{4pt} & \cellcolor{gray!10}\rule{0pt}{4pt} & \cellcolor{gray!10}\rule{0pt}{4pt} & \cellcolor{gray!10}\rule{0pt}{4pt} & \cellcolor{gray!10}\rule{0pt}{4pt} & \rule{0pt}{4pt} & \rule{0pt}{4pt} & \rule{0pt}{4pt} & \rule{0pt}{4pt} & \rule{0pt}{4pt} & \rule{0pt}{4pt} & \rule{0pt}{4pt} & \rule{0pt}{4pt} & \cellcolor{gray!10}\rule{0pt}{4pt} & \cellcolor{gray!10}\rule{0pt}{4pt} & \cellcolor{gray!10}\rule{0pt}{4pt} & \cellcolor{gray!10}\rule{0pt}{4pt} \\[-0.35em]
\multirow[c]{3}{*}{\shortstack[c]{Betting\\Game}} & (50, 0.55) & $[12.1,\,83.7]$ & 2.65 & 0.42 & 0.76 & $[25.2,\,28.5]$ & \textbf{0.12} & 0.90 & 1.17 & $[24.6,\,29.5]$ & 0.18 & 0.58 & 1.42 & $[25.2,\,28.5]$ & \textbf{0.12} & 1.80 & 1.04 & $[24.6,\,29.5]$ & 0.18 & 25.42 & 1.30 & $[25.2,\,28.5]$ & \textbf{0.12} & 0.91 & 0.50 \\
 & (100, 0.55) & $[16.5,\,162]$ & 3.24 & 1.48 & 5.36 & $[41.8,\,47.7]$ & \textbf{0.13} & 4.36 & 6.70 & $[40.9,\,49.5]$ & 0.19 & 2.53 & 8.04 & $[41.8,\,47.7]$ & \textbf{0.13} & 7.96 & 11.16 & $[40.9,\,49.5]$ & 0.19 & 584.54 & 6.64 & $[41.8,\,47.7]$ & \textbf{0.13} & 4.43 & 4.74 \\
 & (150, 0.55) & $[17.5,\,250]$ & 3.91 & 2.40 & 12.99 & $[54.7,\,63.3]$ & \textbf{0.15} & 9.34 & 9.67 & $[53.6,\,66.1]$ & 0.21 & 5.26 & 15.28 & $[54.7,\,63.3]$ & \textbf{0.15} & 15.75 & 14.86 & $[53.6,\,66.1]$ & 0.21 & 1916.81 & 5.58 & $[54.7,\,63.3]$ & \textbf{0.15} & 9.16 & 8.65 \\
\rule{0pt}{4pt} & \rule{0pt}{4pt} & \rule{0pt}{4pt} & \rule{0pt}{4pt} & \rule{0pt}{4pt} & \rule{0pt}{4pt} & \cellcolor{gray!10}\rule{0pt}{4pt} & \cellcolor{gray!10}\rule{0pt}{4pt} & \cellcolor{gray!10}\rule{0pt}{4pt} & \cellcolor{gray!10}\rule{0pt}{4pt} & \cellcolor{gray!10}\rule{0pt}{4pt} & \cellcolor{gray!10}\rule{0pt}{4pt} & \cellcolor{gray!10}\rule{0pt}{4pt} & \cellcolor{gray!10}\rule{0pt}{4pt} & \rule{0pt}{4pt} & \rule{0pt}{4pt} & \rule{0pt}{4pt} & \rule{0pt}{4pt} & \rule{0pt}{4pt} & \rule{0pt}{4pt} & \rule{0pt}{4pt} & \rule{0pt}{4pt} & \cellcolor{gray!10}\rule{0pt}{4pt} & \cellcolor{gray!10}\rule{0pt}{4pt} & \cellcolor{gray!10}\rule{0pt}{4pt} & \cellcolor{gray!10}\rule{0pt}{4pt} \\[-0.35em]
\multirow[c]{3}{*}{Engagement} & (100, 0.30) & $[38.4,\,49.3]$ & 0.26 & 0.07 & 2.40 & $[16.6,\,5422]$ & 127.35 & 0.10 & 1.58 & $[15.3,\,9740]$ & 229.13 & 0.10 & 2.05 & $[39.8,\,45]$ & \textbf{0.12} & 0.13 & 0.46 & $[35.6,\,53.2]$ & 0.41 & 0.13 & 0.49 & $[39.8,\,45]$ & \textbf{0.12} & 0.74 & 0.31 \\
 & (300, 0.30) & $[38.6,\,46.8]$ & 0.19 & 0.16 & 5.92 & $[18.7,\,3615]$ & 84.74 & 0.24 & 6.57 & $[15.7,\,9110]$ & 214.27 & 0.20 & 6.90 & $[40.8,\,45.3]$ & \textbf{0.11} & 0.28 & 3.34 & $[36.7,\,50.6]$ & 0.33 & 0.26 & 3.04 & $[40.8,\,45.3]$ & \textbf{0.11} & 0.86 & 2.28 \\
 & (1000, 0.30) & $[39.1,\,45.4]$ & 0.15 & 0.50 & 20.16 & $[24,\,687]$ & 15.63 & 0.67 & 20.46 & $[16,\,7857]$ & 184.74 & 0.45 & 20.89 & $[41.3,\,43.8]$ & \textbf{0.06} & 0.69 & 20.60 & $[38.5,\,47.7]$ & 0.22 & 0.56 & 19.95 & $[41.3,\,43.8]$ & \textbf{0.06} & 1.28 & 40.77 \\
\rule{0pt}{4pt} & \rule{0pt}{4pt} & \rule{0pt}{4pt} & \rule{0pt}{4pt} & \rule{0pt}{4pt} & \rule{0pt}{4pt} & \cellcolor{gray!10}\rule{0pt}{4pt} & \cellcolor{gray!10}\rule{0pt}{4pt} & \cellcolor{gray!10}\rule{0pt}{4pt} & \cellcolor{gray!10}\rule{0pt}{4pt} & \cellcolor{gray!10}\rule{0pt}{4pt} & \cellcolor{gray!10}\rule{0pt}{4pt} & \cellcolor{gray!10}\rule{0pt}{4pt} & \cellcolor{gray!10}\rule{0pt}{4pt} & \rule{0pt}{4pt} & \rule{0pt}{4pt} & \rule{0pt}{4pt} & \rule{0pt}{4pt} & \rule{0pt}{4pt} & \rule{0pt}{4pt} & \rule{0pt}{4pt} & \rule{0pt}{4pt} & \cellcolor{gray!10}\rule{0pt}{4pt} & \cellcolor{gray!10}\rule{0pt}{4pt} & \cellcolor{gray!10}\rule{0pt}{4pt} & \cellcolor{gray!10}\rule{0pt}{4pt} \\[-0.35em]
\multirow[c]{3}{*}{Mars Rover} & (50, 50) & $[0,\,0.797]$ & 1.59 & 3.22 & 46.72 & $[0.177,\,0.833]$ & 1.31 & 6.26 & 44.63 & $[0.174,\,0.837]$ & 1.32 & 5.03 & 46.19 & $[0.488,\,0.515]$ & \textbf{0.05} & 7.07 & 42.73 & $[0.476,\,0.525]$ & 0.10 & 23.26 & 43.82 & -- & -- & TO & TO \\
 & (75, 75) & $[0,\,0.895]$ & 1.42 & 6.35 & 240.78 & $[0.0904,\,0.9]$ & 1.28 & 19.94 & 232.96 & $[0.0801,\,0.901]$ & 1.30 & 15.01 & 216.61 & $[0.619,\,0.647]$ & \textbf{0.04} & 21.99 & 235.26 & $[0.605,\,0.658]$ & 0.08 & 108.24 & 217.49 & -- & -- & TO & TO \\
 & (125, 125) & $[0,\,0.877]$ & 1.42 & 15.65 & 1279.81 & $[0.0441,\,0.922]$ & 1.42 & 114.78 & 1326.19 & $[0.0422,\,0.922]$ & 1.43 & 57.16 & 1590.10 & $[0.599,\,0.635]$ & \textbf{0.06} & 116.29 & 1384.91 & -- & -- & TO & TO & -- & -- & TO & TO \\
\rule{0pt}{4pt} & \rule{0pt}{4pt} & \rule{0pt}{4pt} & \rule{0pt}{4pt} & \rule{0pt}{4pt} & \rule{0pt}{4pt} & \cellcolor{gray!10}\rule{0pt}{4pt} & \cellcolor{gray!10}\rule{0pt}{4pt} & \cellcolor{gray!10}\rule{0pt}{4pt} & \cellcolor{gray!10}\rule{0pt}{4pt} & \cellcolor{gray!10}\rule{0pt}{4pt} & \cellcolor{gray!10}\rule{0pt}{4pt} & \cellcolor{gray!10}\rule{0pt}{4pt} & \cellcolor{gray!10}\rule{0pt}{4pt} & \rule{0pt}{4pt} & \rule{0pt}{4pt} & \rule{0pt}{4pt} & \rule{0pt}{4pt} & \rule{0pt}{4pt} & \rule{0pt}{4pt} & \rule{0pt}{4pt} & \rule{0pt}{4pt} & \cellcolor{gray!10}\rule{0pt}{4pt} & \cellcolor{gray!10}\rule{0pt}{4pt} & \cellcolor{gray!10}\rule{0pt}{4pt} & \cellcolor{gray!10}\rule{0pt}{4pt} \\[-0.35em]
\multirow[c]{3}{*}{Glider} & (21, 17) & $[41.1,\,44.6]$ & 0.08 & 0.07 & 0.04 & $[42.4,\,42.9]$ & \textbf{0.01} & 0.15 & 0.05 & $[42,\,43.3]$ & 0.03 & 0.23 & 0.04 & $[42.5,\,42.8]$ & \textbf{0.01} & 0.30 & 0.04 & $[42.2,\,43.1]$ & 0.02 & 2.15 & 0.04 & $[42.5,\,42.8]$ & \textbf{0.01} & 21.90 & 0.03 \\
 & (51, 47) & $[89.8,\,22443]$ & 221.40 & 0.17 & 11.20 & $[100.66,\,101.21]$ & 0.01 & 0.82 & 0.27 & $[100,\,102]$ & 0.01 & 0.34 & 0.28 & $[100.76,\,101.12]$ & \textbf{0.00} & 5.97 & 0.27 & $[100.57,\,101.32]$ & 0.01 & 11.47 & 0.27 & $[100.76,\,101.12]$ & \textbf{0.00} & 5782.72 & 0.15 \\
 & (71, 67) & $[87,\,30071]$ & 288.64 & 0.11 & 20.00 & $[103,\,105]$ & \textbf{0.01} & 1.82 & 0.60 & $[102,\,105]$ & 0.03 & 0.66 & 0.65 & $[103,\,104]$ & \textbf{0.01} & 21.50 & 0.61 & $[103,\,105]$ & 0.02 & 49.26 & 0.50 & -- & -- & TO & TO \\
\rule{0pt}{4pt} & \rule{0pt}{4pt} & \rule{0pt}{4pt} & \rule{0pt}{4pt} & \rule{0pt}{4pt} & \rule{0pt}{4pt} & \cellcolor{gray!10}\rule{0pt}{4pt} & \cellcolor{gray!10}\rule{0pt}{4pt} & \cellcolor{gray!10}\rule{0pt}{4pt} & \cellcolor{gray!10}\rule{0pt}{4pt} & \cellcolor{gray!10}\rule{0pt}{4pt} & \cellcolor{gray!10}\rule{0pt}{4pt} & \cellcolor{gray!10}\rule{0pt}{4pt} & \cellcolor{gray!10}\rule{0pt}{4pt} & \rule{0pt}{4pt} & \rule{0pt}{4pt} & \rule{0pt}{4pt} & \rule{0pt}{4pt} & \rule{0pt}{4pt} & \rule{0pt}{4pt} & \rule{0pt}{4pt} & \rule{0pt}{4pt} & \cellcolor{gray!10}\rule{0pt}{4pt} & \cellcolor{gray!10}\rule{0pt}{4pt} & \cellcolor{gray!10}\rule{0pt}{4pt} & \cellcolor{gray!10}\rule{0pt}{4pt} \\[-0.35em]
\multirow[c]{3}{*}{Parallel Betting} & (5, 0.55) & $[24,\,29.4]$ & 0.20 & 0.14 & 0.03 & $[25.1,\,28.2]$ & 0.12 & 0.41 & 0.03 & $[23.1,\,30.6]$ & 0.28 & 0.35 & 0.03 & $[26.1,\,27.2]$ & \textbf{0.04} & 0.49 & 0.03 & $[24.9,\,28.6]$ & 0.14 & 1.71 & 0.03 & $[26.1,\,27.2]$ & \textbf{0.04} & 6.84 & 0.03 \\
 & (10, 0.55) & $[26.4,\,40.6]$ & 0.44 & 1.06 & 0.44 & $[29.7,\,34.5]$ & 0.15 & 2.59 & 0.44 & $[26.9,\,37.9]$ & 0.34 & 1.79 & 0.45 & $[31.1,\,32.9]$ & 0.06 & 3.12 & 0.44 & $[29.3,\,34.9]$ & 0.17 & 25.91 & 0.21 & $[31.1,\,32.8]$ & \textbf{0.05} & 73.99 & 0.21 \\
 & (20, 0.55) & $[13,\,110]$ & 3.39 & 5.56 & 4.81 & $[25.4,\,31.7]$ & \textbf{0.22} & 898.14 & 4.48 & $[22.6,\,35.8]$ & 0.46 & 21.96 & 3.22 & -- & -- & TO & TO & -- & -- & TO & TO & -- & -- & TO & TO \\
\bottomrule
\end{tabular}
} 
\label{tab:bigtable}
\end{table}
\end{landscape}

\end{document}